%% file: paper.tex
\documentclass[journal]{IEEEtran} %
\usepackage[bookmarks=false]{hyperref}
\usepackage{subfig}
\usepackage[utf8]{inputenc} % allow utf-8 input
\usepackage{graphicx,url}
\usepackage{booktabs}       % professional-quality tables
\graphicspath{ {./images/} }
\usepackage{multirow}
\usepackage{hhline}
\usepackage{graphicx}
\usepackage{amssymb,amsmath,mathtools}
\usepackage{amsthm}
\newtheorem{theorem}{Theorem}

\newtheorem{lemma}{Lemma}

\usepackage{bm}                 %

\usepackage{color,soul}
\usepackage[dvipsnames]{xcolor}
\usepackage{bbm}
\usepackage{lipsum,lineno}
\usepackage{float}
\makeatletter
\newif\if@restonecol
\makeatother
\usepackage{amsmath}
\usepackage{mathrsfs}
\usepackage[shortlabels]{enumitem}
\usepackage{algorithm}
\usepackage{algpseudocode}
\usepackage{fancyhdr} % This should be set AFTER setting up the page geometry
\usepackage{tabularx}
\usepackage{dsfont}
\usepackage{comment}               % RS 2020-07-12
\makeatletter
\renewenvironment{proof}[1][\proofname]{\par
  \pushQED{\qed}%
  \normalfont \topsep6\p@\@plus6\p@\relax
  \trivlist
  \item[\hskip\labelsep
        \itshape
    #1]\ignorespaces% ADDED
}{%
  \popQED\endtrivlist\@endpefalse
}
\makeatother
\usepackage{tikz,pgfplots}
\usepackage{pgfplots}
\usetikzlibrary{shapes.geometric,backgrounds,patterns, trees}
\usetikzlibrary{3d,decorations.text,shapes.arrows,positioning,fit,backgrounds}
\usetikzlibrary{positioning, decorations.pathmorphing, shapes}
\usetikzlibrary{decorations.pathreplacing}
\usetikzlibrary{shapes.geometric,backgrounds,patterns, trees}
\usetikzlibrary{arrows.meta,
                bending,
                intersections,
                quotes,
                shapes.geometric}
              \usetikzlibrary{automata, positioning}
\usepgfplotslibrary{fillbetween}
\usetikzlibrary{shapes,arrows}
\usetikzlibrary{arrows.meta}
\usetikzlibrary{positioning}
\tikzset{set/.style={draw,circle,inner sep=0pt,align=center}}
\usetikzlibrary{automata, positioning}
  \usetikzlibrary{shapes,shadows}
  \tikzstyle{abstractbox} = [draw=black, fill=white, rectangle,
  inner sep=10pt, style=rounded corners, drop shadow={fill=black,
  opacity=1}]
\tikzstyle{abstracttitle} =[fill=white]
\usetikzlibrary{calc,positioning,shapes.geometric}
\usetikzlibrary{arrows.meta,arrows}
\DeclareMathOperator*{\argmax}{arg\,max}
\DeclareMathOperator*{\argmin}{arg\,min}

\usetikzlibrary{matrix}
\tikzstyle{cblue}=[circle, draw, thin,fill=cyan!20, scale=0.8]
\tikzstyle{qgre}=[rectangle, draw, thin,fill=green!20, scale=0.8]
\tikzstyle{rpath}=[ultra thick, red, opacity=0.4]
\tikzstyle{legend_isps}=[rectangle, rounded corners, thin,
                       fill=gray!20, text=blue, draw]

\tikzstyle{legend_overlay}=[rectangle, rounded corners, thin,
                           top color= white,bottom color=green!25,
                           minimum width=2.5cm, minimum height=0.8cm,
                           pinegreen]
\tikzstyle{legend_phytop}=[rectangle, rounded corners, thin,
                          top color= white,bottom color=cyan!25,
                          minimum width=2.5cm, minimum height=0.8cm,
                          royalblue]
\tikzstyle{legend_general}=[rectangle, rounded corners, thin,
                          top color= white,bottom color=lavander!25,
                          minimum width=2.5cm, minimum height=0.8cm,
                          violet]
                          \usetikzlibrary{matrix}

\colorlet{myRed}{red!20}
\tikzset{
  rows/.style 2 args={/utils/temp/.style={row ##1/.append style={nodes={#2}}},
    /utils/temp/.list={#1}},
  columns/.style 2 args={/utils/temp/.style={column ##1/.append style={nodes={#2}}},
    /utils/temp/.list={#1}}}
\usetikzlibrary{backgrounds,calc,shadings,shapes.arrows,shapes.symbols,shadows}
\definecolor{switch}{HTML}{006996}

\makeatletter
\pgfkeys{/pgf/.cd,
  parallelepiped offset x/.initial=2mm,
  parallelepiped offset y/.initial=2mm
}
\pgfdeclareshape{parallelepiped}
{
  \inheritsavedanchors[from=rectangle] % this is nearly a rectangle
  \inheritanchorborder[from=rectangle]
  \inheritanchor[from=rectangle]{north}
  \inheritanchor[from=rectangle]{north west}
  \inheritanchor[from=rectangle]{north east}
  \inheritanchor[from=rectangle]{center}
  \inheritanchor[from=rectangle]{west}
  \inheritanchor[from=rectangle]{east}
  \inheritanchor[from=rectangle]{mid}
  \inheritanchor[from=rectangle]{mid west}
  \inheritanchor[from=rectangle]{mid east}
  \inheritanchor[from=rectangle]{base}
  \inheritanchor[from=rectangle]{base west}
  \inheritanchor[from=rectangle]{base east}
  \inheritanchor[from=rectangle]{south}
  \inheritanchor[from=rectangle]{south west}
  \inheritanchor[from=rectangle]{south east}
  \backgroundpath{
    \southwest \pgf@xa=\pgf@x \pgf@ya=\pgf@y
    \northeast \pgf@xb=\pgf@x \pgf@yb=\pgf@y
    \pgfmathsetlength\pgfutil@tempdima{\pgfkeysvalueof{/pgf/parallelepiped
      offset x}}
    \pgfmathsetlength\pgfutil@tempdimb{\pgfkeysvalueof{/pgf/parallelepiped
      offset y}}
    \def\ppd@offset{\pgfpoint{\pgfutil@tempdima}{\pgfutil@tempdimb}}
    \pgfpathmoveto{\pgfqpoint{\pgf@xa}{\pgf@ya}}
    \pgfpathlineto{\pgfqpoint{\pgf@xb}{\pgf@ya}}
    \pgfpathlineto{\pgfqpoint{\pgf@xb}{\pgf@yb}}
    \pgfpathlineto{\pgfqpoint{\pgf@xa}{\pgf@yb}}
    \pgfpathclose
    \pgfpathmoveto{\pgfqpoint{\pgf@xb}{\pgf@ya}}
    \pgfpathlineto{\pgfpointadd{\pgfpoint{\pgf@xb}{\pgf@ya}}{\ppd@offset}}
    \pgfpathlineto{\pgfpointadd{\pgfpoint{\pgf@xb}{\pgf@yb}}{\ppd@offset}}
    \pgfpathlineto{\pgfpointadd{\pgfpoint{\pgf@xa}{\pgf@yb}}{\ppd@offset}}
    \pgfpathlineto{\pgfqpoint{\pgf@xa}{\pgf@yb}}
    \pgfpathmoveto{\pgfqpoint{\pgf@xb}{\pgf@yb}}
    \pgfpathlineto{\pgfpointadd{\pgfpoint{\pgf@xb}{\pgf@yb}}{\ppd@offset}}
  }
}

\makeatletter
\tikzset{anchor/.append code=\let\tikz@auto@anchor\relax,
  add font/.code=%
    \expandafter\def\expandafter\tikz@textfont\expandafter{\tikz@textfont#1},
  left delimiter/.style 2 args={append after command={\tikz@delimiter{south east}
    {south west}{every delimiter,every left delimiter,#2}{south}{north}{#1}{.}{\pgf@y}}}}
\tikzstyle{sms} = [rectangle callout, draw,very thick, rounded corners, minimum height=20pt]
\makeatletter
\tikzset{anchor/.append code=\let\tikz@auto@anchor\relax,
  add font/.code=%
    \expandafter\def\expandafter\tikz@textfont\expandafter{\tikz@textfont#1},
  left delimiter/.style 2 args={append after command={\tikz@delimiter{south east}
    {south west}{every delimiter,every left delimiter,#2}{south}{north}{#1}{.}{\pgf@y}}}}
\tikzstyle{sms} = [rectangle callout, draw,very thick, rounded corners, minimum height=20pt]
\usetikzlibrary{positioning,calc}

\tikzset{l3 switch/.style={
    parallelepiped,fill=switch, draw=white,
    minimum width=0.75cm,
    minimum height=0.75cm,
    parallelepiped offset x=1.75mm,
    parallelepiped offset y=1.25mm,
    path picture={
      \node[fill=white,
        circle,
        minimum size=6pt,
        inner sep=0pt,
        append after command={
          \pgfextra{
            \foreach \angle in {0,45,...,360}
            \draw[-latex,fill=white] (\tikzlastnode.\angle)--++(\angle:2.25mm);
          }
        }
      ]
       at ([xshift=-0.75mm,yshift=-0.5mm]path picture bounding box.center){};
    }
  },
  ports/.style={
    line width=0.3pt,
    top color=gray!20,
    bottom color=gray!80
  },
  rack switch/.style={
    parallelepiped,fill=white, draw,
    minimum width=1.25cm,
    minimum height=0.25cm,
    parallelepiped offset x=2mm,
    parallelepiped offset y=1.25mm,
    xscale=-1,
    path picture={
      \draw[top color=gray!5,bottom color=gray!40]
      (path picture bounding box.south west) rectangle
      (path picture bounding box.north east);
      \coordinate (A-west) at ([xshift=-0.2cm]path picture bounding box.west);
      \coordinate (A-center) at ($(path picture bounding box.center)!0!(path
        picture bounding box.south)$);
      \foreach \x in {0.275,0.525,0.775}{
        \draw[ports]([yshift=-0.05cm]$(A-west)!\x!(A-center)$)
          rectangle +(0.1,0.05);
        \draw[ports]([yshift=-0.125cm]$(A-west)!\x!(A-center)$)
          rectangle +(0.1,0.05);
       }
      \coordinate (A-east) at (path picture bounding box.east);
      \foreach \x in {0.085,0.21,0.335,0.455,0.635,0.755,0.875,1}{
        \draw[ports]([yshift=-0.1125cm]$(A-east)!\x!(A-center)$)
          rectangle +(0.05,0.1);
      }
    }
  },
  server/.style={
    parallelepiped,
    fill=white, draw,
    minimum width=0.35cm,
    minimum height=0.75cm,
    parallelepiped offset x=3mm,
    parallelepiped offset y=2mm,
    xscale=-1,
    path picture={
      \draw[top color=gray!5,bottom color=gray!40]
      (path picture bounding box.south west) rectangle
      (path picture bounding box.north east);
      \coordinate (A-center) at ($(path picture bounding box.center)!0!(path
        picture bounding box.south)$);
      \coordinate (A-west) at ([xshift=-0.575cm]path picture bounding box.west);
      \draw[ports]([yshift=0.1cm]$(A-west)!0!(A-center)$)
        rectangle +(0.2,0.065);
      \draw[ports]([yshift=0.01cm]$(A-west)!0.085!(A-center)$)
        rectangle +(0.15,0.05);
      \fill[black]([yshift=-0.35cm]$(A-west)!-0.1!(A-center)$)
        rectangle +(0.235,0.0175);
      \fill[black]([yshift=-0.385cm]$(A-west)!-0.1!(A-center)$)
        rectangle +(0.235,0.0175);
      \fill[black]([yshift=-0.42cm]$(A-west)!-0.1!(A-center)$)
        rectangle +(0.235,0.0175);
    }
  },
}
\pgfplotsset{compat=1.16}
\usetikzlibrary{calc, shadings, shadows, shapes.arrows}

\tikzset{%
  interface/.style={draw, rectangle, rounded corners, font=\LARGE\sffamily},
  ethernet/.style={interface, fill=yellow!50},% ethernet interface
  serial/.style={interface, fill=green!70},% serial interface
  speed/.style={sloped, anchor=south, font=\large\sffamily},% line speed at edge
  route/.style={draw, shape=single arrow, single arrow head extend=4mm,
    minimum height=1.7cm, minimum width=3mm, white, fill=switch!20,
    drop shadow={opacity=.8, fill=switch}, font=\tiny}% inroute/outroute arrows
}
\newcommand*{\shift}{1.3cm}% For placing the arrows later
\newcommand*{\router}[1]{
\begin{tikzpicture}
  \coordinate (ll) at (-3,0.5);
  \coordinate (lr) at (3,0.5);
  \coordinate (ul) at (-3,2);
  \coordinate (ur) at (3,2);
  \shade [shading angle=90, left color=switch, right color=white] (ll)
    arc (-180:-60:3cm and .75cm) -- +(0,1.5) arc (-60:-180:3cm and .75cm)
    -- cycle;
  \shade [shading angle=270, right color=switch, left color=white!50] (lr)
    arc (0:-60:3cm and .75cm) -- +(0,1.5) arc (-60:0:3cm and .75cm) -- cycle;
  \draw [thick] (ll) arc (-180:0:3cm and .75cm)
    -- (ur) arc (0:-180:3cm and .75cm) -- cycle;
  \draw [thick, shade, upper left=switch, lower left=switch,
    upper right=switch, lower right=white] (ul)
    arc (-180:180:3cm and .75cm);
  \node at (0,0.5){\color{blue!60!black}\Huge #1};% The name of the router
  \begin{scope}[yshift=2cm, yscale=0.28, transform shape]
    \node[route, rotate=45, xshift=\shift] {\strut};
    \node[route, rotate=-45, xshift=-\shift] {\strut};
    \node[route, rotate=-135, xshift=\shift] {\strut};
    \node[route, rotate=135, xshift=-\shift] {\strut};
  \end{scope}
\end{tikzpicture}}

\makeatletter
\pgfdeclareradialshading[tikz@ball]{cloud}{\pgfpoint{-0.275cm}{0.4cm}}{%
  color(0cm)=(tikz@ball!75!white);
  color(0.1cm)=(tikz@ball!85!white);
  color(0.2cm)=(tikz@ball!95!white);
  color(0.7cm)=(tikz@ball!89!black);
  color(1cm)=(tikz@ball!75!black)
}
\tikzoption{cloud color}{\pgfutil@colorlet{tikz@ball}{#1}%
  \def\tikz@shading{cloud}\tikz@addmode{\tikz@mode@shadetrue}}
\makeatother

\tikzset{my cloud/.style={
     cloud, draw, aspect=2,
     cloud color={gray!5!white}
  }
}

\renewcommand\thetable{\arabic{table}}
\usepackage{etoolbox}
\makeatletter
\patchcmd{\@makecaption}
  {\scshape}
  {}
  {}
  {}
\makeatother

\begin{document}
\bstctlcite{MyBSTcontrol}
\title{Learning Security Strategies \\through Game Play and Optimal Stopping}

\author{\IEEEauthorblockN{Kim Hammar \IEEEauthorrefmark{2}\IEEEauthorrefmark{3} and Rolf Stadler\IEEEauthorrefmark{2}\IEEEauthorrefmark{3}}

 \IEEEauthorblockA{\IEEEauthorrefmark{2}
Division of Network and Systems Engineering, KTH Royal Institute of Technology, Sweden
 }\\
 \IEEEauthorblockA{\IEEEauthorrefmark{3} KTH Center for Cyber Defense and Information Security, Sweden \\
Email: \{kimham, stadler\}@kth.se%
\\
\today
}
}
%%\thanks{© 2021 IEEE. Personal use of this material is permitted.}

%%\markboth{\copyright 2021 IEEE; This work has been submitted to the IEEE for possible publication. Copyright may be transferred without notice.}%
%%{}
\maketitle
\begin{abstract}
We study automated intrusion prevention using reinforcement learning. Following a novel approach, we formulate the interaction between an attacker and a defender as an optimal stopping game and let attack and defense strategies evolve through reinforcement learning and self-play. The game-theoretic perspective allows us to find defender strategies that are effective against dynamic attackers. The optimal stopping formulation gives us insight into the structure of optimal strategies, which we show to have threshold properties. To obtain the optimal defender strategies, we introduce \textsc{T-FP}, a fictitious self-play algorithm that learns Nash equilibria through stochastic approximation. We show that \textsc{T-FP} outperforms a state-of-the-art algorithm for our use case. Our overall method for learning and evaluating strategies includes two systems: a simulation system where defender strategies are incrementally learned and an emulation system where statistics are produced that drive simulation runs and where learned strategies are evaluated. We conclude that this approach can produce effective defender strategies for a practical IT infrastructure.
\end{abstract}


\IEEEpeerreviewmaketitle

\section{Introduction}
An organization's security strategy has traditionally been defined, implemented, and updated by domain experts \cite{int_prevention}. Although this approach can provide basic security for an organization's communication and computing infrastructure, a growing concern is that infrastructure update cycles become shorter and attacks increase in sophistication \cite{cyber_threat_landscape}. Consequently, the security requirements become increasingly difficult to meet. To address this challenge, significant efforts have started to automate security frameworks and the process of obtaining security strategies. Examples of this research include: automated creation of threat models \cite{mal_pontus}; computation of defender strategies using dynamic programming and control theory \cite{dp_security_1,Miehling_control_theoretic_approaches_summary}; computation of exploits and corresponding defenses through evolutionary methods \cite{armsrace_malware,hemberg_oreily_evo}; identification of infrastructure vulnerabilities through attack simulations and threat intelligence \cite{wagner_automated_segmentation, threat_intel_misp}; computation of defender strategies through game-theoretic methods \cite{nework_security_alpcan, serkan_gyorgy_game}; and use of machine learning techniques to estimate model parameters and strategies \cite{hammar_stadler_tnsm, hammar_stadler_cnsm_20, hammar_stadler_cnsm_21}.

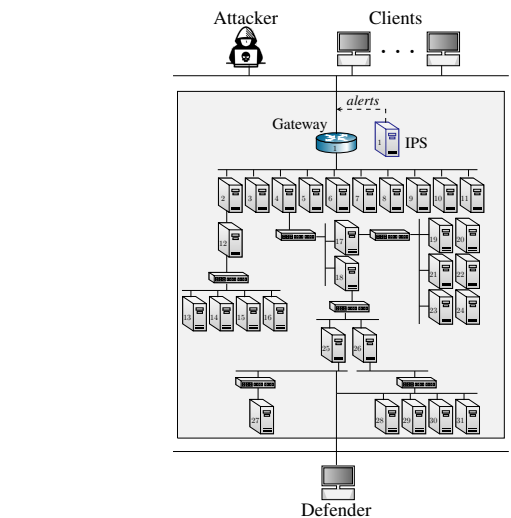
\begin{figure}
  \centering
  \scalebox{0.93}{
    \input{tikz/system5.tex}
    }
    \caption{The IT infrastructure and the actors in the use case.}
    \label{fig:system2}
  \end{figure}

In this paper, we present a novel approach to automatically learn security strategies for an attacker and a defender. We apply this approach to an \textit{intrusion prevention} use case, which involves the IT infrastructure of an organization (see Fig. \ref{fig:system2}). The operator of this infrastructure, which we call the defender, takes measures to protect it against a possible attacker while providing services to a client population. (We use the term "intrusion prevention'' as suggested in the literature, e.g. in \cite{int_prevention}, it means that the attacker is prevented from reaching its goal, rather than prevented from accessing any part of the infrastructure.)

We formulate the use case as an \textit{optimal stopping game}, i.e. a stochastic game where each player faces an optimal stopping problem \cite{wald,dynkin_orig_3,shirayev_change_point}. The stopping game formulation enables us to gain insight into the structure of optimal strategies, which we show to have threshold properties. To obtain effective defender strategies, we use reinforcement learning and self-play. Based on the threshold properties of optimal strategies, we design an efficient self-play algorithm that iteratively computes optimal defender strategies against a dynamic attacker.

Our method for learning and evaluating strategies includes building two systems (see Fig. \ref{fig:method}). First, we develop an \textit{emulation system} where key functional components of the target infrastructure are replicated. This system closely approximates the functionality of the target infrastructure and is used to run attack scenarios and defender responses. These runs produce system measurements and logs from which we estimate distributions of infrastructure metrics. We then use the estimated distributions to instantiate the simulation model. Second, we build a \textit{simulation system} where game episodes are simulated and strategies are incrementally learned. Learned strategies are then extracted from the simulation system and evaluated in the emulation system. (A video demonstration of our software framework that implements the emulation and simulation systems is available at \cite{fig2_video_demonstration}.)
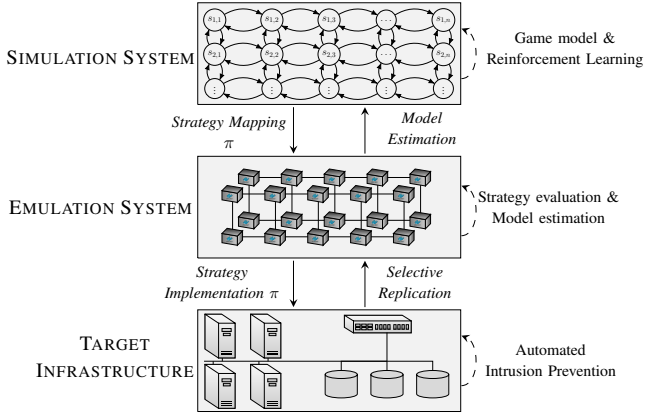
\begin{figure}
  \centering
  \scalebox{0.95}{
    \input{tikz/method.tex}
 }
    \caption{Our approach for finding and evaluating intrusion prevention strategies.}
    \label{fig:method}
\end{figure}

We make three contributions with this paper. First, we formulate intrusion prevention as an optimal stopping game. This novel formulation allows us a) to derive structural properties of optimal strategies using results from optimal stopping theory; and b) to find defender strategies that are effective against attackers with dynamic strategies. We thus address a limitation of many related works that consider static attackers only \cite{oslo_pentest_rl,ridley_ml_defense,deep_hierarchical_rl_pentest,pentest_rl_rohit,deep_air,hammar_stadler_cnsm_21,adaptive_cyber_defense_pomdp_rl,kurt_rl}. Second, we propose \textsc{T-FP}, an efficient reinforcement learning algorithm that exploits structural properties of optimal stopping strategies and outperforms a state-of-the-art algorithm for our use case. Third, we provide evaluation results from an emulated infrastructure. This addresses a common drawback in related research, which relies on simulations to learn and evaluate strategies \cite{hammar_stadler_cnsm_20,hammar_stadler_cnsm_21, elderman, schwartz_2020, oslo_pentest_rl, kurt_rl, microsoft_red_teaming, ridley_ml_defense, rl_cyberdefense_heartbleed, deep_hierarchical_rl_pentest, pentest_rl_rohit, adaptive_cyber_defense_pomdp_rl,serkan_gyorgy_game,flipit,dynamic_game_linan_zhu,general_sum_markov_games_for_strategic_detection_of_apt,game_cyber_rl_sim,cmu_ppo_selfplay,al_shaer_book_ppo_simulation,mec_game_rl_q_learning_sim,nfsp_jamming_1_sim}.

\section{The Intrusion Prevention Use Case}\label{sec:use_case}
We consider an intrusion prevention use case that involves the IT infrastructure of an organization. The operator of this infrastructure, which we call the defender, takes measures to protect it against an attacker while providing services to a client population (Fig. \ref{fig:system2}). The infrastructure includes a set of servers that run the services and an Intrusion Prevention System (IPS) that logs events in real-time. Clients access the services through a public gateway, which also is open to the attacker.

The attacker's goal is to intrude on the infrastructure and compromise its servers. To achieve this, the attacker explores the infrastructure through reconnaissance and exploits vulnerabilities while avoiding detection by the defender. The attacker decides when to start an  intrusion and may stop the intrusion at any moment. During the intrusion, the attacker follows a pre-defined strategy. When deciding the time to start or stop an intrusion, the attacker considers both the gain of compromising additional servers and the risk of getting detected. The optimal strategy for the attacker is to compromise as many servers as possible without being detected.

The defender continuously monitors the infrastructure through accessing and analyzing IPS alerts and other statistics. It can take a fixed number of defensive actions, each of which has a cost and a chance of preventing an ongoing attack. An example of a defensive action is to drop network traffic that triggers IPS alerts of a certain priority. The defender takes defensive actions in a pre-determined order, starting with the action that has the lowest cost. The final action blocks all external access to the gateway, which disrupts any ongoing intrusion as well as the services to the clients.

When deciding the time for taking a defensive action, the defender balances two objectives: (\textit{i}) maintain services to its clients; and (\textit{ii}), prevent a possible intrusion at lowest cost. The optimal strategy for the defender is to monitor the infrastructure and maintain services until the moment when the attacker enters through the gateway, at which time the attack must be prevented at minimal cost through defensive actions. The challenge for the defender is to identify the precise time for this moment.

\section{Formalizing The Intrusion Prevention \\Use Case}\label{sec:formal_model_2}
We model the use case as a partially observed stochastic game. The attacker wins the game when it can intrude on infrastructure and hide its actions from the defender. In contrast, the defender wins the game when it manages to prevent an intrusion. We model this as a zero-sum game, which means that the gain of one player equals the loss of the other player.

The attacker and the defender have different observability in the game. The defender observes alerts from an Intrusion Prevention System (IPS) but has no certainty about the presence of an attacker or the state of a possible intrusion. The attacker, on the other hand, is assumed to have complete observability. It has access to all the information that the defender has access to, as well as the defender's past actions. The asymmetric observability requires the defender to find strategies that are effective against any attacker, including attackers with inside information about its monitoring capabilities.

The reward function of the game encodes the defender's objective. An optimal defender strategy \textit{maximizes} reward when facing a worst-case attacker. Similarly, an optimal attacker strategy \textit{minimizes} reward when facing a worst-case defender. In game-theoretical terms, this means that a pair of optimal strategies is a Nash equilibrium \cite{nash51}.

%\subsection{The Intrusion Prevention Game}\label{sec:formalization_game}
Formally, we model the game as a finite and zero-sum Partially Observed Stochastic Game (POSG) with one-sided partial observability: $\Gamma = \langle \mathcal{N}, \mathcal{S}, (\mathcal{A}_i)_{i \in \mathcal{N}},$ $\mathcal{T}, (\mathcal{R}_{i})_{i \in \mathcal{N}}, \gamma,  \rho_1, T, (\mathcal{O}_i)_{i \in \mathcal{N}}, \mathcal{Z} \rangle$. It is a discrete-time game that starts at time $t=1$. In the following, we describe the components of the game, its evolution, and the objectives of the players.

\textbf{Players $\mathcal{N}$.} The game has two players: player $1$ is the defender and player $2$ is the attacker. Hence, $\mathcal{N}=\{1,2\}$.

\textbf{State space $\mathcal{S}$.} The game has three states: $s_t=0$ if no intrusion is occurring, $s_t=1$ if an intrusion is ongoing, and $s_t=\emptyset$ if the game has ended. Hence, $\mathcal{S}=\{0,1,\emptyset\}$. The initial state is $s_1=0$ and the initial state distribution is the degenerate distribution $\rho_1(0)=1$.

\textbf{Action spaces $\mathcal{A}_i$.} Each player $i\in \mathcal{N}$ can invoke two actions: ``stop'' ($S$) and ``continue'' ($C$). The action spaces are thus $\mathcal{A}_1=\mathcal{A}_2=\{S,C\}$. $S$ results in a change of state and $C$ means that the game remains in the same state. We encode $S$ with $1$ and $C$ with $0$.

The attacker can invoke the stop action two times: the first time to start the intrusion and the second time to stop it. The defender can invoke the stop action $L \geq 1$ times. Each stop of the defender can be interpreted as a defensive action against a possible intrusion. The number of stops remaining of the defender at time-step $t$ is known to both the attacker and the defender and is denoted by $l_t \in \{1,\hdots,L\}$.

At each time-step, the attacker and the defender simultaneously choose an action each: $\bm{a}_t = (a^{(1)}_t, a^{(2)}_t)$, where $a^{(i)}_t \in \mathcal{A}_i$.

\textbf{Observation space $\mathcal{O}$.} The attacker has complete information and knows the game state, the defender's actions, and the defender's observations. The defender, however, only sees the observations $o_t \in \mathcal{O}$, where $\mathcal{O}$ is a discrete set. (In our use case, $o_t$ relates to the number of IPS alerts during time-step $t$.)

Both players have perfect recall, meaning that they remember their respective play history. The history of the defender at time-step $t$ is $h^{(1)}_t$$=(\rho_1$, $a^{(1)}_{1}$, $o_{1}$, $\hdots$, $a^{(1)}_{t-1}$, $o_{t})$ and the history of the attacker is $h^{(2)}_t$$=(\rho_1$, $a^{(1)}_{1}$, $a^{(2)}_{1}$, $o_{1}$, $s_1$,$\hdots$, $a^{(1)}_{t-1}$, $a^{(2)}_{t-1}$, $o_{t}$, $s_{t})$.

\textbf{Belief space $\mathcal{B}$.} Based on its history $h^{(1)}_t$, the defender forms a belief about $s_t$, which is expressed in the \textit{belief state} $b_t(s_t)=\mathbb{P}[s_t|h^{(1)}_t] \in \mathcal{B}$. Since $s_{t} \in \{0,1\}$ and $b_{t}(0) = 1-b_t(1)$, $b_{t}$ is determined by $b_{t}(1)$. Hence, we can model $\mathcal{B} = [0,1]$.

\textbf{Transition probabilities $\mathcal{T}$.} At each time-step $t$, a state transition occurs. The probabilities of the state transitions are defined by $\mathcal{T}_{l_t}$$\big($$s_{t+1}$, $s_t$, $(a^{(1)}_t$, $a^{(2)}_t)\big)$ $= \mathbb{P}_{l_t}$$\big[$$s_{t+1}|$ $s_t$, $(a^{(1)}_t, a^{(2)}_t)$$\big]$:
\begin{align}
&\mathcal{T}_{l_t>1}\big(0,0,(S,C)\big)=\mathcal{T}_{l_t}\big(0,0,(C,C)\big)=1\label{eq:tp_1}\\
&\mathcal{T}_{l_t>1}\big(1,1,(\cdot,C)\big)=\mathcal{T}_{l_t}\big(1,1,(C,C)\big) = 1-\phi_{l_t}  \label{eq:tp_2}\\
&\mathcal{T}_{l_t>1}\big(1,0,(\cdot,S)\big)=\mathcal{T}_{l_t}\big(1,0,(C,S)\big) = 1\label{eq:tp_3}\\
&\mathcal{T}_{l_t>1}\big(\emptyset,1,(\cdot,C)\big)=\mathcal{T}_{l_t}\big(\emptyset,1,(C,C)\big)=\phi_{l_t}\label{eq:tp_4}\\
&\mathcal{T}_{1}\big(\emptyset,\cdot,(S,\cdot)\big)=\mathcal{T}_{l_t}(\emptyset,\emptyset,\cdot)=\mathcal{T}_{l_t}(\emptyset,1,(\cdot, S))=1\label{eq:tp_7}
\end{align}
All other state transitions have probability $0$.

Eqs. \ref{eq:tp_1}-\ref{eq:tp_2} define the probabilities of the recurrent state transitions $0\rightarrow 0$ and $1\rightarrow 1$. The game stays in state $0$ with probability $1$ if the attacker selects action $C$ and $l_t-a_t^{(1)}>0$. Similarly, the game stays in state $1$ with probability $1-\phi_{l_t}$ if the attacker chooses action $C$ and $l_t-a_t^{(1)}>0$. $\phi_{l_t}$ is a parameter of the use case that defines the probability that the intrusion is prevented, which increases with each stop action that the defender takes.

Eq. \ref{eq:tp_3} captures the transition $0 \rightarrow 1$, which occurs when the attacker chooses action $S$ and $l_t-a_t^{(1)}>0$. Eqs. \ref{eq:tp_4}-\ref{eq:tp_7} define the probabilities of the transitions to the terminal state $\emptyset$. The terminal state is reached in three cases: (\textit{i}) when $l_t=1$ and the defender takes the final stop action $S$ (i.e. when $l_t-a^{(1)}_t=0$); (\textit{ii}) when the intrusion is prevented with probability $\phi_{l_t}$; and \textit{(iii)}, when $s_t=1$ and the attacker stops ($a^{(2)}_t=1$).

The evolution of the game can be described with the state transition diagram in Fig. \ref{fig:state_transitions}. The figure describes a game \textit{episode}, which starts at $t=1$ and ends at $t=T$. The time horizon $T$ is a random variable that depends on both players' strategies and takes values in the set $\{t=2,3,\hdots,\infty\}$.

\begin{figure}
  \centering
  \scalebox{1.05}{
    \input{tikz/state_transitions_9.tex}
    }
    \caption{State transition diagram of a game episode: each disk represents a state; an arrow represents a state transition; a label indicates the conditions for the state transition; a game episode starts in state $s_1=0$ with $l_1=L$ and ends in state $s_T=\emptyset$.}
    \label{fig:state_transitions}
  \end{figure}
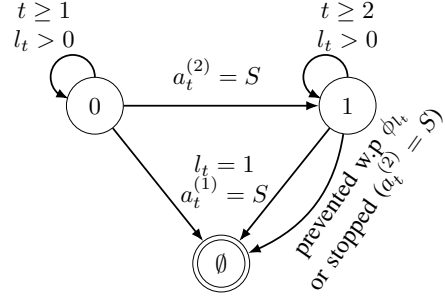

\textbf{Reward function $\mathcal{R}_{l_t}$.} At time-step $t$, the defender receives the reward $r_t = \mathcal{R}_{l_t}(s_t,(a^{(1)}_t, a^{(2)}_t))$ and the attacker receives the reward $-r_t$. The reward function is parameterized by the reward that the defender receives for stopping an intrusion ($R_{st} > 0$), the defender's cost of taking a defensive action ($R_{cost}< 0$), and its loss when being intruded ($R_{int}< 0$):
\begin{align}
&\mathcal{R}_{l_t}(\emptyset, \cdot)=0, \quad \quad \mathcal{R}_{l_t}\big(1, (\cdot,S)\big) = 0 \label{eq:reward_0}\\
&\mathcal{R}_{l_t}\big(0, (C,\cdot)\big)=0 \label{eq:reward_2}\\
&\mathcal{R}_{l_t}\big(0, (S,\cdot)\big)=R_{cost}/l_t    \label{eq:reward_3}\\
&\mathcal{R}_{l_t}\big(1, (S,C)\big)=R_{st}/l_t  \label{eq:reward_4}\\
&\mathcal{R}_{l_t}\big(1, (C,C)\big)=R_{int}  \label{eq:reward_5}
\end{align}
Eq. \ref{eq:reward_0} states that the reward is zero in the terminal state and when the attacker ends an intrusion. Eq. \ref{eq:reward_2} states that the defender incurs no cost when it is not under attack and not taking defensive actions. Eq. \ref{eq:reward_3} indicates that the defender incurs a cost when stopping if no intrusion is ongoing, which is decreasing with the number of stops remaining $l_t$. Eq. \ref{eq:reward_4} states that the defender receives a reward that is decreasing in $l_t$ when taking a stop action that affects an ongoing intrusion. Lastly, Eq. \ref{eq:reward_5} indicates that the defender receives a loss for each time-step when under intrusion.

\textbf{Observation function $\mathcal{Z}$.} At time-step $t$, $o_t \in \mathcal{O}$ is drawn from a random variable $O$ whose distribution $f_{O}$ depends on the current state $s_{t}$. We define $\mathcal{Z}(o_{t},s_t,(a^{(1)}_{t-1}, a^{(2)}_{t-1}))$$=\mathbb{P}[o_t|s_t,(a^{(1)}_{t-1}, a^{(2)}_{t-1})]$ as follows:
\begin{align}
&\mathcal{Z}\big(o_t,0,\cdot \big) = f_{O}(o_t |0) \label{eq:obs_1}\\
&\mathcal{Z}\big(o_t,1,\cdot \big) = f_{O}(o_t |1)\label{eq:obs_2}\\
&\mathcal{Z}\big(\emptyset,\emptyset,\cdot\big) = 1 \label{eq:obs_3}
\end{align}\normalsize

\textbf{Belief update.} At time-step $t$, the belief state $b_{t}$ is updated as follows:
\begin{align}
&b_{t+1}(s_{t+1}) = C \sum_{s_t \in \mathcal{S}}\sum_{a^{(2)}_t \in \mathcal{A}_2}\sum_{o_{t+1} \in \mathcal{O}}b_t(s_t)\pi_{2,l}(a^{(2)}_{t}|s_t,b_t)\cdot \nonumber\\
  &\mathcal{Z}(o_{t+1},s_{t+1},(a^{(1)}_{t},a^{(2)}_t)) \mathcal{T}\big(s_{t+1},s_t,(a^{(1)}_{t},a^{(2)}_t)\big)\label{eq:belief_upd}
\end{align}
where $C=1/\mathbb{P}[o_{t+1}|a^{(1)}_1,\pi_{2,l}, b_t]$ is a normalizing factor that makes $b_{t+1}$ sum to $1$. The initial belief is $b_1(0)=1$.

\textbf{Player strategies $\pi_{i,l}$.} A strategy of the defender is a function $\pi_{1,l} \in \Pi_1: \mathcal{B} \rightarrow \Delta(\mathcal{A}_1)$. Analogously, a strategy of the attacker is a function $\pi_{2,l} \in \Pi_2: \mathcal{S} \times \mathcal{B} \rightarrow \Delta(\mathcal{A}_2)$. $\Delta(\mathcal{A}_i)$ denotes the set of probability distributions over $\mathcal{A}_i$, $\Pi_i$ denotes the strategy space of player $i$, and $\pi_{-i,l}$ denotes the strategy of player $j\in \mathcal{N}\setminus \{i\}$. For both players, a strategy is dependent on $l$ but independent of $t$, i.e. strategies are stationary. If $\pi_{i,l}$ always maps on to an action with probability $1$, it is called \textit{pure}, otherwise it is called \textit{mixed}.

\textbf{Objective functions $J_i$.} The goal of the defender is to \textit{maximize} the expected discounted cumulative reward over the time horizon $T$. Similarly, the goal of the attacker is to \textit{minimize} the same quantity. Therefore, the objective functions $J_1$ and $J_2$ are:
\begin{align}
J_1(\pi_{1,l}, \pi_{2,l}) &= \mathbb{E}_{(\pi_{1,l}, \pi_{2,l})}\left[\sum_{t=1}^{T}\gamma^{t-1}\mathcal{R}_{l_t}(s_t, \bm{a}_t)\right] \label{eq:objective_1}\\
J_2(\pi_{1,l}, \pi_{2,l}) &= -J_1(\pi_{1,l}, \pi_{2,l}) \label{eq:objective_2}
\end{align}
where $\gamma \in [0,1)$ is the discount factor.
%0.99

\textbf{Best response strategies $\tilde{\pi}_{i,l}$.} A defender strategy $\tilde{\pi}_{1,l} \in B_1(\pi_{2,l})$ is called a \textit{best response} against $\pi_{2,l}\in \Pi_{2}$ if it \textit{maximizes} $J_1$ (Eq. \ref{eq:br_defender}). Similarly, an attacker strategy $\tilde{\pi}_{2,l} \in B_2(\pi_{1,l})$ is called a best response against $\pi_{1,l} \in \Pi_1$ if it \textit{minimizes} $J_1$ (Eq. \ref{eq:br_attacker}).
\begin{align}
B_1(\pi_{2,l}) &= \argmax_{\pi_{1,l} \in \Pi_1}J_1(\pi_{1,l}, \pi_{2,l})\label{eq:br_defender}\\
B_2(\pi_{1,l}) &= \argmin_{\pi_{2,l} \in \Pi_2}J_1(\pi_{1,l}, \pi_{2,l})\label{eq:br_attacker}
\end{align}

\textbf{Optimal strategies $\pi^{*}_{i,l}$.} An optimal defender strategy $\pi_{1,l}^{*}$ is a best response strategy against any attacker strategy that \textit{minimizes} $J_1$. Similarly, an optimal attacker strategy $\pi_{2,l}^{*}$ is a best response against any defender strategy that \textit{maximizes} $J_1$. Hence, when both players follow optimal strategies, they play best response strategies against each other:
\begin{align}
(\pi_{1,l}^{*}, \pi_{2,l}^{*}) \in B_1(\pi_{2,l}^{*}) \times B_2(\pi_{1,l}^{*})\label{eq:minmax_objective}
\end{align}
This means that no player has an incentive to change its strategy and that $(\pi_{1,l}^{*},\pi_{2,l}^{*})$ is a Nash equilibrium \cite{nash51}.

\section{Game-Theoretic Analysis and Our Approach for Finding Optimal Defender Strategies}\label{sec:game_analysis}
  Finding optimal strategies that satisfy Eq. \ref{eq:minmax_objective} means finding a Nash equilibrium for the POSG $\Gamma$. We know from game theory that $\Gamma$ has at least one mixed Nash equilibrium \cite{vonNeumann_1928:TGG,nash51, Shapley1095,posg_equilibria_existence_finite_horizon}. (A Nash equilibrium is called mixed if one or more players follow mixed strategies.)

The equilibria of $\Gamma$ can be obtained by finding pairs of strategies that are best responses against each other (Eq. \ref{eq:minmax_objective}). A best response for the defender is obtained by solving a POMDP $\mathcal{M}^{P}$, and a best response for the attacker is obtained by solving an MDP $\mathcal{M}$. Hence, the best response strategies can be expressed with $Q$-functions:
\begin{align}
&B_1(\pi_{2,l})=\argmax_{a^{(1)}_t \in \mathcal{A}_1} Q_{1,\pi_{2,l}}^*(b_t,a^{(1)}_t)\\
&B_2(\pi_{1,l})=\argmin_{a^{(2)}_t \in \mathcal{A}_2} Q_{2,\pi_{1,l}}^*((b_t,s_t),a^{(2)}_t)
\end{align}
The corresponding Bellman equations are \cite{bellman1957markovian}:
\begin{align}
&Q_{i,\pi}^{*}(x_t,a^{(i)}_t) = \mathop{\mathbb{E}}_{\pi,x_t,a^{(i)}_t}\big[r_{t+1} + \gamma V_{i,\pi}^{*}(x_{t+1})\big] \label{eq:bellman_eq_41}\\
  &V_{1,\pi}^{*}(x_t) = \max_{a^{(1)}_t\in \mathcal{A}_1} \mathop{\mathbb{E}}_{\pi,x_t,a^{(1)}_t}\big[r_{t+1} + \gamma V_{1,\pi}^{*}(x_{t+1})\big]\label{eq:bellman_eq_42}
\end{align}
\begin{align}
  &V_{2,\pi}^{*}(x_t) = \min_{a^{(2)}_t\in \mathcal{A}_2} \mathop{\mathbb{E}}_{\pi,a^{(2)}_t}\big[r_{t+1} + \gamma V_{2,\pi}^{*}(x_{t+1})\big]\label{eq:bellman_eq_43}\\
&V^{*}(b) = \max_{\pi_{1,l} \in \Delta(\mathcal{A}_1)}\min_{\pi_{2,l} \in \Delta(\mathcal{A}_2)} \mathbb{E}_{\pi_{1,l},\pi_{2,l},b}\big[r^{(1)}_{t+1} + \gamma V^{*}(b_{t+1})\big]  \label{eq:bellman_posg_1}
\end{align}
\normalsize
Since the game is zero-sum, stationary, and $\gamma < 1$, it follows from game theory that $V^{*}(b) = V_{1,\pi^{*}_{2,l}}^{*}(b) = V_{2,\pi^{*}_{1,l}}^{*}(b,s)$ \cite{vonNeumann_1928:TGG,horak_solving_one_sided_posgs}. Further, from Markov decision theory we know that for any strategy pair ($\pi_{1,l}, \pi_{2,l}$), a corresponding pair of best response strategies $(\tilde{\pi}_{1,l} \in B_1(\pi_{2,l}), \tilde{\pi}_{2,l} \in B_2(\pi_{1,l}))$ exists \cite{puterman}.
\begin{figure}
  \centering
    \scalebox{1.3}{
      \includegraphics{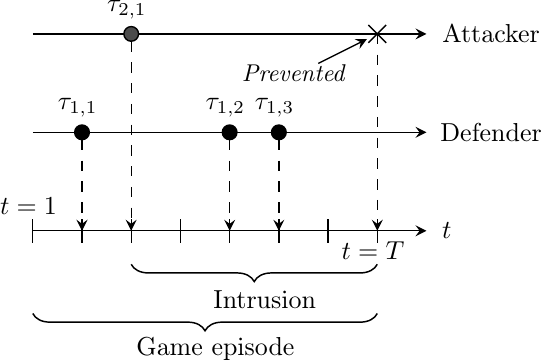}
    }
    \caption{Stopping times of the defender and the attacker in a game episode; the lower horizontal axis represents time; the black circles on the middle axis and the upper axis represent time-steps of defender stop actions and attacker stop actions, respectively; $\tau_{i,j}$ denotes the $j$th stopping time of player $i$; the cross shows the time the intrusion is prevented; an episode ends either when the attacker is prevented or when it takes its second stop action.}
    \label{fig:stopping_times}
  \end{figure}
\subsection{Analyzing Best Responses using Optimal Stopping Theory}
The POMDP $\mathcal{M}^{P}$ and the MDP $\mathcal{M}$ that determine the best response strategies can be understood as \textit{optimal stopping} problems (see Fig. \ref{fig:stopping_times}) \cite{wald,stopping_book_1,chow1971great,hammar_stadler_tnsm}. In the defender's case, the problem is to find a stopping strategy $\pi_{1,l}^{*}(b_t) \rightarrow \{S,C\}$ that maximizes $J_1$ (Eq. \ref{eq:objective_1}) and prescribes the optimal stopping times $\tau^{*}_{1,1},\tau^{*}_{1,2},\hdots, \tau^{*}_{1,L}$. Similarly, the problem for the attacker is to find a stopping strategy $\pi_{2,l}^{*}(s_t,b_t) \rightarrow \{S,C\}$ that minimizes $J_1$ (Eq. \ref{eq:objective_2}) and prescribes the optimal stopping times $\tau^{*}_{2,1}$ and $\tau^{*}_{2,2}$.

Given a pair of stopping strategies $(\pi_{1,l},\pi_{2,l})$ and their best responses $(\tilde{\pi}_{1,l}\in B_1(\tilde{\pi}_{2,l}), \tilde{\pi}_{2,l}\in B_2(\tilde{\pi}_{1,l}))$, we define two subsets of $\mathcal{B}$: the \textit{stopping sets} and the \textit{continuation sets}.

The stopping sets contain the belief states where $S$ is a best response: $\mathscr{S}^{(1)}_{l,\pi_{2,l}} = \{b(1) \in [0,1] : \tilde{\pi}_{1,l}\big(b(1)\big) = S\}$ and $\mathscr{S}^{(2)}_{s,l,\pi_{1,l}} = \{b(1) \in [0,1] : \tilde{\pi}_{2,l}\big(s,b(1)\big) = S\}$. Similarly, the continuation sets contain the belief states where $C$ is a best response: $\mathscr{C}^{(1)}_{l,\pi_{2,l}} = \{b(1) \in [0,1] : \tilde{\pi}_{1,l}\big(b(1)\big) = C\}$ and $\mathscr{C}^{(2)}_{s,l,\pi_{1,l}} = \{b(1) \in [0,1] : \tilde{\pi}_{2,l}\big(s, b(1)\big) = C\}$.

Based on \cite{krishnamurthy_2016,Nakai1985,optimal_multiple_stopping_social_media_1,hammar_stadler_tnsm,horak_solving_one_sided_posgs}, we formulate Theorem \ref{thm:best_responses}, which contains an existence result for equilibria and a structural result for best response strategies in the game.
\begin{theorem}\label{thm:best_responses}
Given the one-sided POSG $\Gamma$ in Section \ref{sec:formal_model_2} with $L \geq 1$, the following holds.
\begin{enumerate}[(A)]
\item $\Gamma$ has a mixed Nash equilibrium. Further, $\Gamma$ has a pure Nash equilibrium when $s=0 \iff b(1)=0$.
\item
Given any attacker strategy $\pi_{2,l} \in \Pi_2$, if the probability mass function $f_{O|s}$ is totally positive of order 2 (i.e., TP2 \cite[Definition 10.2.1, pp. 223]{krishnamurthy_2016}), there exist values $\tilde{\alpha}_{1}$ $\geq$ $\tilde{\alpha}_{2}$ $\geq$ $\hdots$ $\geq$ $\tilde{\alpha}_L \in [0,1]$ and a best response strategy $\tilde{\pi}_{1,l} \in B_1(\pi_{2,l})$ of the defender that satisfies:
\begin{align}
\tilde{\pi}_{1,l}(b(1)) = S \iff b(1) \geq \tilde{\alpha}_l \quad\quad l\in 1,\hdots,L \label{eq:prop_br_defender}
\end{align}
\item Given a defender strategy $\pi_{1,l}\in \Pi_1$, where $\pi_{1,l}(S|b(1))$ is non-decreasing in $b(1)$ and $\pi_{1,l}(S|1)=1$, there exist values $\tilde{\beta}_{0,1},$ $\tilde{\beta}_{1,1},$ $\hdots$, $\tilde{\beta}_{0,L}$, $\tilde{\beta}_{1,L} \in [0,1]$ and a best response strategy $\tilde{\pi}_{2,l} \in B_2(\pi_{1,l})$ of the attacker that satisfies:
  \begin{align}
\tilde{\pi}_{2,l}(0,b(1)) = C \iff \pi_{1,l}(S|b(1)) \geq \tilde{\beta}_{0,l} \label{eq:prop_br_attacker_1}\\
\tilde{\pi}_{2,l}(1,b(1)) = S \iff \pi_{1,l}(S|b(1)) \geq \tilde{\beta}_{1,l} \label{eq:prop_br_attacker_2}
\end{align}
for $l \in 1,\hdots, L$.
\end{enumerate}
\end{theorem}
\begin{proof}[Proof.]
See Appendices \ref{appendix:theorem_1_a}-\ref{appendix:theorem_1_c}.
\end{proof}
\begin{figure}
  \centering
  \scalebox{1.15}{
    \input{tikz/threshold_policy_8.tex}
  }
  \caption{Illustration of Theorem \ref{thm:best_responses}; the upper plot shows the existence of $L$ thresholds $\tilde{\alpha}_{1} \geq \tilde{\alpha}_{2}, \hdots, \geq \tilde{\alpha}_{L} \in [0,1]$ that define a best response strategy $\tilde{\pi}_{1,l,\tilde{\theta}^{(1)}} \in B_1(\pi_{2,l})$ (Eq. \ref{eq:prop_br_defender}); the lower plot shows the existence of $2L$ thresholds $\tilde{\beta}_{0,1}, \tilde{\beta}_{1,1}, \hdots, \tilde{\beta}_{0,L}, \tilde{\beta}_{1,L} \in [0,1]$ that define a best response strategy $\tilde{\pi}_{2,l,\tilde{\theta}^{(2)}} \in B_2(\pi_{1,l})$ (Eqs. \ref{eq:prop_br_attacker_1}-\ref{eq:prop_br_attacker_2}).}
    \label{fig:threshold_policy_3}
  \end{figure}
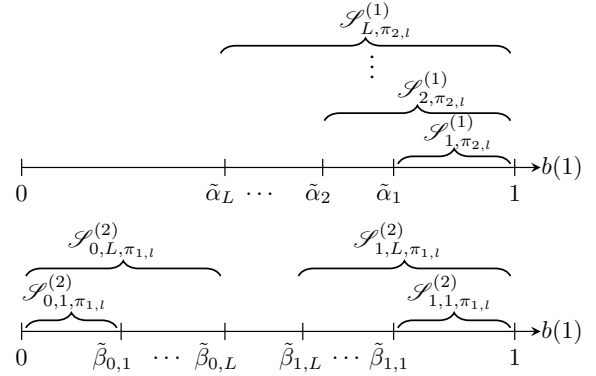

Theorem \ref{thm:best_responses} tells us that $\Gamma$ has a mixed Nash equilibrium. It also tells us that, under certain assumptions, the best response strategies have threshold properties (see Fig. \ref{fig:threshold_policy_3}). In the following, we describe an efficient algorithm that takes advantage of these properties to approximate Nash equilibria of $\Gamma$.

\subsection{Finding Nash Equilibria through Fictitious Self-Play}\label{sec:rl_approach}
Computing Nash equilibria for a POSG is generally intractable \cite{horak_solving_one_sided_posgs}. However, approximate solutions can be obtained through iterative approximation methods. One such method is \textit{fictitious self-play}, where both players start from random strategies and continuously update their strategies based on the outcomes of played game episodes \cite{brown_fictious_play}.

Fictitious self-play evolves through a sequence of iteration steps, which is illustrated in Fig. \ref{fig:fp_2}. An iteration step includes three procedures. First, player $1$ learns a best response strategy against player $2$'s current strategy. The roles are then reversed and player $2$ learns a best response strategy against player $1$'s current strategy. Lastly, the iteration step is completed by having each player adopt a new strategy, which is determined by the empirical distribution over its past best response strategies. The sequence of iteration steps continues until the strategies of both players have sufficiently converged to a Nash equilibrium \cite{brown_fictious_play,multiagent_systems_book_1}.
\subsection{Our Self-Play Algorithm: \textsc{T-FP}}\label{sec:t_fp}
We present a fictitious self-play algorithm, which we call \textsc{T-FP}, that exploits the statements in Theorem \ref{thm:best_responses} to efficiently approximate Nash equilibria of $\Gamma$.

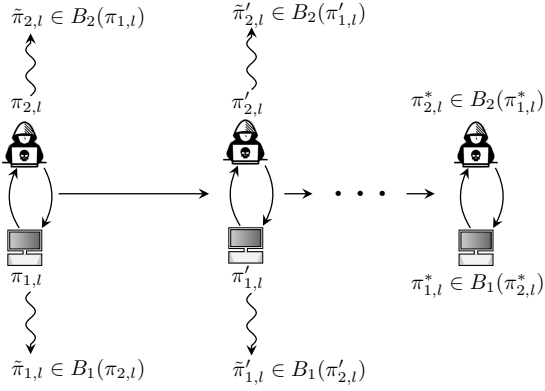
\begin{figure}
  \centering
  \scalebox{1.25}{
    \input{tikz/fp_3.tex}
  }
  \caption{The fictitious self-play process; in every iteration each player learns a best response strategy $\tilde{\pi}_{i,l} \in B_i(\pi_{{-i},l})$ and updates its strategy based on the empirical distribution of its past best responses; the horizontal arrows indicate the iterations of self-play and the vertical arrows indicate the learning of best responses; if the process is convergent, it reaches a Nash equilibrium $(\pi_{1,l}^{*},\pi_{2,l}^{*})$.} \label{fig:fp_2}
\end{figure}
\textsc{T-FP} implements the fictitious self-play process described in Section \ref{sec:rl_approach} and generates a sequence of strategy profiles $(\pi_{1,l}, \pi_{2,l})$, $(\pi^{\prime}_{1,l}$, $\pi^{\prime}_{2,l})$, $\hdots$, $(\pi^{*}_{1,l}, \pi^{*}_{2,l})$ that converges to a Nash equilibrium. During each step of this process, \textsc{T-FP} learns best responses against the players' current strategies and then updates the strategies of both players to be the empirical distribution over the past strategies (see Fig. \ref{fig:fp_2}).

\textsc{T-FP} parameterizes the best response strategies $\tilde{\pi}_{1,l, \tilde{\theta}^{(1)}}\in B_1(\pi_{2,l})$ and $\tilde{\pi}_{2,l,\tilde{\theta}^{(2)}} \in B_2(\pi_{1,l})$ by threshold vectors. The defender's best response strategy is parameterized with the vector $\tilde{\theta}^{(1)} \in \mathbb{R}^{L}$ (Eq. \ref{eq:smooth_threshold}). Similarly, the attacker's best response strategy is parameterized with the vector $\tilde{\theta}^{(2)} \in \mathbb{R}^{2L}$ (Eq. \ref{eq:smooth_threshold_2}).
\begin{align}
&\varphi(a,b) = \left(1 + \left(\frac{b(1-\sigma(a))}{\sigma(a)(1-b)}\right)^{-20}\right)^{-1} \label{eq:differentiable_smooth}
\end{align}
\begin{align}
&\tilde{\pi}_{1,l,\tilde{\theta}^{(1)}}\big(S|b(1)\big) = \varphi\left(\tilde{\theta}^{(1)}_l, b(1)\right) \label{eq:smooth_threshold}
\end{align}
\begin{align}
&\tilde{\pi}_{2,l,\tilde{\theta}^{(2)}}\big(S|b(1),s\big) = \varphi\left(\tilde{\theta}^{(2)}_{sL+l}, \pi_{1,l}(S|b(1))\right)\label{eq:smooth_threshold_2}
\end{align}
%\begin{align}
%\varphi(a,b) &= \left(1 + \left(\frac{b(1-\sigma(a))}{\sigma(a)(1-b)}\right)^{-20}\right)^{-1} \label{eq:differentiable_smooth}\\
%\tilde{\pi}_{1,l,\tilde{\theta}^{(1)}}\big(S|b(1)\big) &= \varphi\left(\tilde{\theta}^{(1)}_l, b(1)\right) \label{eq:smooth_threshold}\\
%\tilde{\pi}_{2,l,\tilde{\theta}^{(2)}}\big(S|b(1),s\big) &= \varphi\left(\tilde{\theta}^{(2)}_{sL+l}, \pi_{1,l}(S|b(1))\right)\label{eq:smooth_threshold_2}
%\end{align}
$\sigma(\cdot)$ is the sigmoid function, $\sigma(\tilde{\theta}^{(1)}_{1})$, $\sigma(\tilde{\theta}^{(1)}_{2})$, $\hdots$, $\sigma(\tilde{\theta}^{(1)}_{L}) \in [0,1]$ are the $L$ thresholds of the defender (see Theorem \ref{thm:best_responses}.B), and $\sigma(\tilde{\theta}^{(2)}_{1})$, $\sigma(\tilde{\theta}^{(2)}_{2})$, $\hdots$, $\sigma(\tilde{\theta}^{(2)}_{2L}) \in [0,1]$ are the $2L$ thresholds of the attacker (see Theorem \ref{thm:best_responses}.C).

Using this parameterization, \textsc{T-FP} learns best response strategies by iteratively updating $\tilde{\theta}^{(1)}$ and $\tilde{\theta}^{(2)}$ through stochastic approximation. To update the threshold vectors, \textsc{T-FP} simulates $\Gamma$, which allows to evaluate the objective functions $J_1(\tilde{\pi}_{1,l,\tilde{\theta}^{(1)}}, \pi_{2,l})$ (Eq. \ref{eq:objective_1}) and $J_2(\pi_{1,l}, \tilde{\pi}_{2,l,\tilde{\theta}^{(2)}})$ (Eq. \ref{eq:objective_2}). The obtained values of $J_1$ and $J_2$ are then used to estimate the gradients $\nabla_{\tilde{\theta}^{(1)}}J_1$ and $\nabla_{\tilde{\theta}^{(2)}}J_2$ using the Simultaneous Perturbation Stochastic Approximation (SPSA) gradient estimator \cite{spsa, spsa_impl}. Next, the estimated gradients are used to update $\tilde{\theta}^{(1)}$ and $\tilde{\theta}^{(2)}$ through stochastic gradient ascent. This procedure of estimating gradients and updating $\tilde{\theta}^{(1)}$ and $\tilde{\theta}^{(2)}$ continues until $\tilde{\pi}_{1,l,\tilde{\theta}^{(1)}}$ and $\tilde{\pi}_{2,l,\tilde{\theta}^{(2)}}$ have sufficiently converged.

After the best response strategies $\tilde{\pi}_{1,l,\tilde{\theta}^{(1)}}$ and $\tilde{\pi}_{2,l,\tilde{\theta}^{(2)}}$ have converged, the threshold vectors are added to buffers $\Theta^{(1)}$ and $\Theta^{(2)}$, which contain the vectors learned in previous iterations of \textsc{T-FP}. The iteration step is completed by having both players update their strategies based on the empirical distributions over the past vectors in the buffer.

%The sequences of iteration steps continues until the strategies have sufficiently converged to a Nash equilibrium.

The pseudocode of \textsc{T-FP} is listed in Algorithm \ref{alg:ne_approximation}. (Here $\mathcal{U}_{k}(\{-1,1\})$ denotes a $k$-dimensional discrete multivariate uniform distribution on $\{-1,1\}$.)

\begin{algorithm}
  \caption{\textsc{T-FP}}\label{alg:ne_approximation}
  \hspace*{\algorithmicindent} \textbf{Input} \\
  \hspace*{\algorithmicindent}  $\Gamma, N$: the POSG and $\#$ best response iterations\\
%  \hspace*{\algorithmicindent}  $\pi_{1,l},\pi_{2,l}$ initial strategies \\
  \hspace*{\algorithmicindent}  $a,c,\lambda,A,\epsilon,\delta$: scalar coefficients\\
  \hspace*{\algorithmicindent} \textbf{Output} \\
  \hspace*{\algorithmicindent}  $(\pi^{*}_{1,l}, \pi^{*}_{2,l})$: an approximate Nash equilibrium
\begin{algorithmic}[1]
  \Procedure{T-FP}{}
  \State $\tilde{\theta}^{(1)} \sim \mathcal{U}_L(\{-1,1\})$, $\quad \tilde{\theta}^{(2)} \sim \mathcal{U}_{2L}(\{-1,1\})$
  \State $\Theta^{(1)} \leftarrow \{\tilde{\theta}^{(1)}\}, \quad \Theta^{(2)} \leftarrow \{\tilde{\theta}^{(2)}\}, \quad \hat{\delta} \leftarrow \infty$
  \State $\pi_{1,l}\leftarrow \text{\textsc{EmpiricalDistribution}($\Theta^{(1)}$)}$
  \State $\pi_{2,l}\leftarrow \text{\textsc{EmpiricalDistribution}($\Theta^{(2)}$)}$
  \While{$\hat{\delta} \geq \delta$}
  \For{$i \in \{1,2\}$}
  \State $\tilde{\theta}_{(1)}^{(i)} \sim \mathcal{U}_{iL}(\{-1,1\})$
  \For{$n \in \{1, \hdots, N\}$}
  \State $a_n \leftarrow \frac{a}{(n + A)^{\epsilon}}, \quad c_n \leftarrow \frac{c}{n^{\lambda}}$
  \For{$k \in \{1, \hdots, iL$}
  \State $(\Delta_n)_k \sim \mathcal{U}_1(\{-1,1\})$
  \EndFor
  \State $R_{high} \sim J_i(\pi_{i,l,\tilde{\theta}^{(i)}_{(n)}} + c_n\Delta_n, \pi_{-i,l})$
  \State $R_{low} \sim J_i(\pi_{i,l,\tilde{\theta}^{(i)}_{(n)}} - c_n\Delta_n, \pi_{-i,l})$
  \For{$k \in \{1, \hdots, iL\}$}
  \State $G \leftarrow \frac{R_{high} - R_{low}}{2c_n(\Delta_n)_{k}}$
  \State $\left(\hat{\nabla}_{\tilde{\theta}^{(i)}_{(n)}}J_i(\pi_{i,l,\tilde{\theta}^{(i)}_{(n)}}, \pi_{-i,l})\right)_{k} \leftarrow G$
  \EndFor
  \State $\tilde{\theta}^{(i)}_{(n+1)} = \tilde{\theta}^{(i)}_{(n)} + a_n\hat{\nabla}_{\tilde{\theta}^{(i)}_{(n)}}J_i(\pi_{i,l,\tilde{\theta}^{(i)}_{(n)}},\pi_{-i,l})$
  \EndFor
  \State $\Theta^{(i)} \leftarrow \Theta^{(i)} \cup \tilde{\theta}^{(i)}_{(N+1)}$
  \EndFor
  \State $\pi_{1,l}\leftarrow \text{\textsc{EmpiricalDistribution}($\Theta^{(1)}$)}$
  \State $\pi_{2,l}\leftarrow \text{\textsc{EmpiricalDistribution}($\Theta^{(2)}$)}$
  \State $\hat{\delta} = \text{\textsc{Exploitability}}(\pi_{1,l}, \pi_{2,l})$
  \EndWhile
  \State \Return $(\pi_{1,l}, \pi_{2,l})$
\EndProcedure
\end{algorithmic}
\end{algorithm}
\section{Emulating the Target Infrastructure to Instantiate the Simulation}\label{sec:policy_learning_results}
To simulate a game episode we must know the observation distribution conditioned on the system state (see Eqs. \ref{eq:obs_1}-\ref{eq:obs_3}). We estimate this distribution using measurements from the emulation system shown in Fig. \ref{fig:method}. Moreover, to evaluate the performance of strategies learned in the simulation system, we run game episodes in the emulation system by having the attacker and the defender take actions at the times prescribed by the learned stopping strategies.
\subsection{Emulating the Target Infrastructure}\label{sec:emu_target_inf}
The emulation system executes on a cluster of machines that runs a virtualization layer provided by Docker \cite{docker} containers and virtual links. The system implements network isolation and traffic shaping on the containers using network namespaces and the NetEm module in the Linux kernel \cite{netem}. Resource constraints on the containers, e.g. CPU and memory constraints, are enforced using cgroups.

The network topology of the emulated infrastructure is given in Fig. \ref{fig:system2} and the configuration is given in Appendix \ref{appendix:infrastructure_configuration}. The system emulates the clients, the attacker, the defender, network connectivity, and $31$ physical components of the target infrastructure (e.g application servers and the gateway). The software functions replicate important components of the target infrastructure, such as, web servers, databases, and the Snort IPS, which is deployed using Snort's community ruleset v2.9.17.1.

We emulate connections between servers as full-duplex loss less connections with capacity $1$ Gbit/s in both directions. We emulate external connections between the gateway and the client population as full-duplex connections of $100$ Mbit/s capacity and $0.1\%$ packet loss with random bursts of $1\%$ packet loss. (These numbers are drawn from empirical studies on enterprise and wide area networks \cite{packet_losses_decreasing,Paxson97end-to-endinternet,elliott_markov_chain_ref}.)

\subsection{Emulating the Client Population}\label{sec:emu_target_inf}
The \textit{client population} is emulated by processes that run inside Docker containers and interact with the application servers through the gateway. The clients select functions uniformly at random from the list given in Table \ref{tab:client_profiles}. We emulate client arrivals using a stationary Poisson process with parameter $\lambda=20$ and exponentially distributed service times with parameter $\mu=\frac{1}{4}$. The duration of a time-step in the emulation is $30$ seconds.
%We assume no waiting times, i.e. clients are served at arrival.
\begin{table}
\centering
%\resizebox{0.85\columnwidth}{!}{%
\begin{tabular}{ll} \toprule
  {\textit{Functions}} & {\textit{Application servers}} \\ \midrule
  HTTP, SSH, SNMP, ICMP & $N_2,N_3,N_{10},N_{12}$\\
  IRC, PostgreSQL, SNMP & $N_{31},N_{13},N_{14},N_{15},N_{16}$\\
  FTP, DNS, Telnet & $N_{10}, N_{22}, N_{4}$ \\
  \bottomrule\\
\end{tabular}
%%$}
\caption{Emulated client population; each client interacts with application servers using a set of network functions.}\label{tab:client_profiles}
\end{table}
\subsection{Emulating Defender and Attacker Actions}\label{sec:emu_player_actions}
The attacker and the defender observe the infrastructure continuously and take actions at discrete time-steps $t=1,2,\hdots, T$. During each step, the defender and the attacker can perform one action each.

The defender executes either a continue action or a stop action. Only the stop action affects the progression of the emulation. We have implemented $L=7$ stop actions, which are listed in Table \ref{tab:defender_stop_actions}. The first stop action revokes user certificates and recovers user accounts thought to be compromised by the attacker. The second stop action updates the firewall configuration of the gateway to drop traffic from IP addresses that have been flagged by the IPS. Stop actions $3-6$ update the configuration of the IPS to drop traffic that generates alerts of priorities $1-4$. The final stop action blocks all incoming traffic. (Contrary to Snort's terminology, we define $4$ to be the highest priority.)

\begin{table}
  \centering
\resizebox{1\columnwidth}{!}{%
\begin{tabular}{ll} \toprule
  {\textit{Stop index}} & {\textit{Action}} \\ \midrule
  $1$ & Revoke user certificates \\
  $2$ & Blacklist IPs \\
  $3$ & Drop traffic that generates IPS alerts of priority $1$ \\
  $4$ & Drop traffic that generates IPS alerts of priority $2$ \\
  $5$ & Drop traffic that generates IPS alerts of priority $3$ \\
  $6$ & Drop traffic that generates IPS alerts of priority $4$ \\
  $7$ & Block gateway \\
  \bottomrule\\
\end{tabular}
}
\caption{Defender stop actions in the emulation.}\label{tab:defender_stop_actions}
\end{table}
\begin{table}
\centering
\begin{tabular}{ll} \toprule
  {\textit{Type}} & {\textit{Actions}} \\ \midrule
  Reconnaissance  & TCP-SYN scan, UDP port scan, \\
                  & TCP Null scan, TCP Xmas scan, TCP FIN scan, \\
                  & ping-scan, TCP connection scan, \\
                  & ``Vulscan'' vulnerability scanner \\
  &\\
  Brute-force attack & Telnet, SSH, FTP, Cassandra,\\
                  &  IRC, MongoDB, MySQL, SMTP, Postgres\\
                  &\\
  Exploit & CVE-2017-7494, CVE-2015-3306,\\
                  & CVE-2010-0426, CVE-2015-5602, \\
                  &  CVE-2014-6271, CVE-2016-10033\\
                  & CVE-2015-1427, SQL Injection\\
  \bottomrule\\
\end{tabular}
%%}
\caption{Attacker commands to emulate intrusions.}\label{tab:attacker_actions}
\end{table}

Like the defender, the attacker executes either a stop action or a continue action during each time-step. The attacker can take two stop actions. The first determines when the intrusion starts and the second determines when it ends (see Section \ref{sec:formal_model_2}). A continue action in state $s=0$ has no affect on the emulation, but a continue action in state $s=1$ has. When the attacker takes a stop action in state $s=0$ or a continue action in state $s=1$, an intrusion command is executed. We have implemented $25$ such commands, which are listed in Table \ref{tab:attacker_actions}. During each step of an intrusion, the attacker selects a command uniformly at random from the list in Table \ref{tab:attacker_actions}.
\subsection{Estimating the IPS Alert Distribution}\label{sec:estimating_dist}
At the end of every time-step, the emulation system collects the metric $o_t$, which contains the number of IPS alerts that occurred during the time-step, weighted by priority. For the evaluation reported in this paper we collect measurements from $23000$ time-steps of $30$ seconds each.
\begin{figure}
  \centering
    \scalebox{0.83}{
      \includegraphics{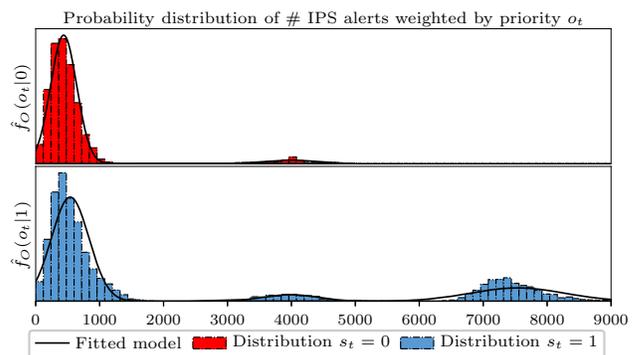}
    }
    \caption{Empirical distributions of $o_t$ when no intrusion occurs ($s_t=0$) and during intrusion ($s_t=1$); the black lines show the fitted Gaussian mixture models.}
    \label{fig:ids_distribution}
\end{figure}

Using these measurements, we fit a Gaussian mixture distribution $\hat{f}_{O}$ as an estimate of $f_{O}$ in the target infrastructure (Eqs. \ref{eq:obs_1}-\ref{eq:obs_2}). For each state $s$, we obtain the conditional distribution $\hat{f}_{O | s}$ through expectation-maximization \cite{em_demp_77}.

Fig. \ref{fig:ids_distribution} shows the empirical distributions and the fitted model over the discrete observation space $\mathcal{O} = \{1,2,\hdots,9000\}$. $\hat{f}_{O | 0}$ and $\hat{f}_{O | 1}$ are Gaussian mixtures with two and three components, respectively. Both mixtures have most probability mass within the range $0-1000$. $\hat{f}_{O | 1}$ also has substantial probability mass at larger values.

The stochastic matrix with the rows $\hat{f}_{O | 0}$ and $\hat{f}_{O | 1}$ has about $ 72 \times 10^{6}$ minors, out of which virtually all are non-negative. This suggests to us that the TP2 assumption in Theorem \ref{thm:best_responses} can be made.
\subsection{Running Game Episodes}\label{sec:simulation_episode}
During a simulation, the game state evolves according to the dynamics described in Section \ref{sec:formal_model_2} and the defender's belief state evolves according to Eq. \ref{eq:belief_upd}. The actions of both players are determined by their strategies, and the observations are sampled from the estimated observation distribution $\hat{f}_O$.

An episode in the emulation system differs from an episode in the simulation system. First, the emulated client population issues requests to the emulated application servers (see Section \ref{sec:emu_target_inf}). Second, the defender's observations are not sampled but are obtained through reading log files and metrics of the emulated infrastructure, which depend on the network traffic generated by the client population, the attacker, as well as internal infrastructure processes. Third, attacker and defender actions in the emulation system include executing networking and computing functions (see Table \ref{tab:defender_stop_actions} and Table \ref{tab:attacker_actions}).

%To select the attacker and defender actions at each time-step,
We collect the observations in the emulation system using a distributed log implemented with Kafka \cite{kafka} (see Fig. \ref{fig:management_1}). Log updates are read periodically by a program that computes the defender's belief state $b(1)$ and the game state $s$. Using this information, the attacker's and the defender's strategies determine the next actions, which are executed in the emulation system using an API implemented over gRPC \cite{grpc} and SSH.
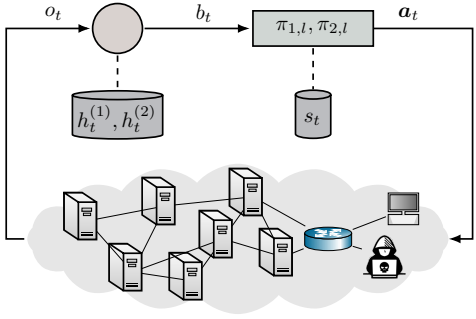
\begin{figure}
  \centering
  \scalebox{1}{
    \input{tikz/management_2.tex}
  }
  \caption{Emulation of a game episode; measurement data ($o_t$) is aggregated in a log that is consumed by a stream processor to compute the next belief $b_t$ based on the history $h_t$; the next pair of actions $\bm{a}_t$ is sampled from the strategy pair $(\pi_{1,l}, \pi_{2,l})$ and is executed in the infrastructure using a gRPC/SSH API.}
    \label{fig:management_1}
\end{figure}
\section{Learning Nash Equilibrium Strategies for the Target Infrastructure}\label{sec:eval}
Our approach to finding effective defender strategies includes: (1) extensive simulation of game episodes in the simulation system to learn Nash equilibrium strategies; and (2) evaluation of the learned strategies on the emulation system (see Fig. \ref{fig:method}). This section describes our evaluation results for the intrusion prevention use case.

The environment for running simulations and training strategies is a Tesla P100 GPU. The hyperparameters for the training algorithm are listed in Appendix \ref{appendix:hyperparameters}. The emulated infrastructure is deployed on a server with a 24-core Intel Xeon Gold 2.10 GHz CPU and 768 GB RAM.

The code for the simulation system and the measurement traces for the intrusion prevention use case are available at \cite{github_cnsm_21_hammar_stadler}. They can be used to validate our results and extend this research.
\subsection{Learning Equilibrium Strategies through Self-Play}
\begin{figure*}
\centering
\scalebox{0.94}{
      \includegraphics{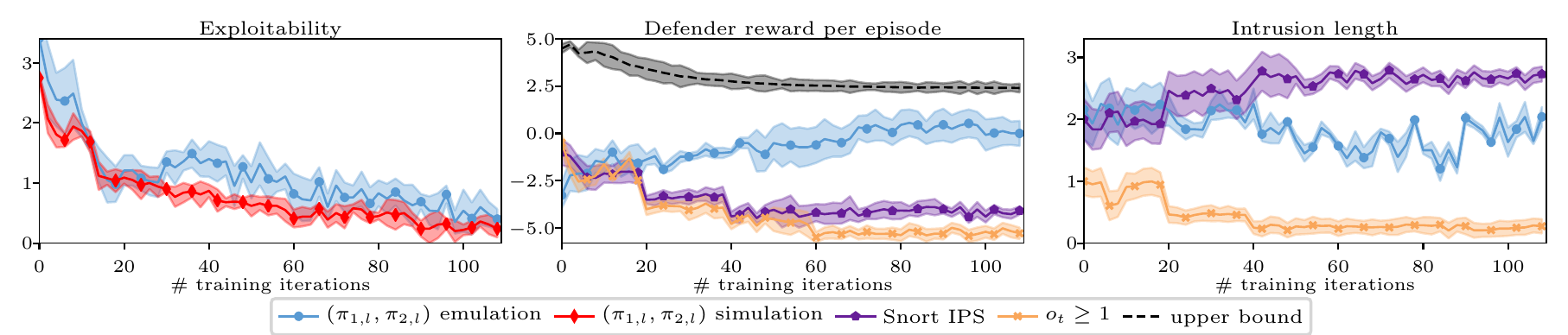}
}
\caption{Learning curves from the self-play process with \textsc{T-FP}; the red curve show simulation results and the blue curves show emulation results; the purple, orange, and black curves relate to baseline strategies; the figures show different performance metrics: exploitability, episodic reward, and the length of intrusion; the curves indicate the mean and the $95\%$ confidence interval over four training runs with different random seeds.}
    \label{fig:exploitability_curve}
  \end{figure*}
We run \textsc{T-FP} for $500$ iterations to estimate a Nash equilibrium using the iterative method described in Section \ref{sec:rl_approach}. At the end of each iteration step, we evaluate the current strategy pair $(\pi_{1,l}, \pi_{2,l})$ by running $500$ evaluation episodes in the simulation system and $5$ evaluation episodes in the emulation system. This allows us to produce learning curves for different performance metrics (see Fig. \ref{fig:exploitability_curve}).

To estimate the convergence of the sequence of strategy pairs to a Nash equilibrium, we use the \textit{approximate exploitability} metric $v^{exp}$ \cite{approx_br}:
\begin{align}
v^{exp} = J_1(\hat{\pi}_{1,l}, \pi_{2,l}) + J_2(\pi_{1,l}, \hat{\pi}_{2,l})
\end{align}
where $\hat{\pi}_{i,l}$ denotes an approximate best response strategy for player $i$ obtained through dynamic programming. The closer $v^{exp}$ becomes to $0$, the closer $(\pi_{1,l},\pi_{2,l})$ is to a Nash equilibrium.

The $500$ training iterations constitute one \textit{training run}. We run four training runs with different random seeds. A single training run takes about $5$ hours of processing time on a P100 GPU. In addition, it takes around $12$ hours to evaluate the strategies on the emulation system.

\textbf{Defender baseline strategies.} We compare the learned defender strategies with three baselines. The first baseline prescribes the stop action whenever an IPS alert occurs, i.e., whenever $o_t \geq 1$. The second baseline follows the Snort IPS's internal recommendation system and takes a stop action whenever $100$ IP packets have been dropped by the Snort IPS (see Appendix \ref{appendix:infrastructure_configuration} for the Snort configuration). The third baseline assumes knowledge of the exact intrusion time and performs all stop actions at subsequent time-steps.

\textbf{Baseline algorithms.} We compare the performance of \textsc{T-FP} with two baseline algorithms: Neural Fictitious Self-Play (NFSP) \cite{heinrich_1} and Heuristic Search Value Iteration (HSVI) for one-sided POSGs \cite{horak_bosansky_hsvi}. NFSP is a state-of-the-art deep reinforcement learning algorithm for imperfect-information games. Similar to \textsc{T-FP}, NFSP is a fictitious self-play algorithm. However, contrary to \textsc{T-FP}, NFSP does not exploit the threshold structures expressed in Theorem \ref{thm:best_responses} and as a result is more complex. HSVI is a state-of-the-art dynamic programming algorithm for one-sided POSGs.
\subsection{Discussion of the Evaluation Results}\label{sec:one_stop_evaluation}
\begin{figure}
\centering
\scalebox{0.86}{
      \includegraphics{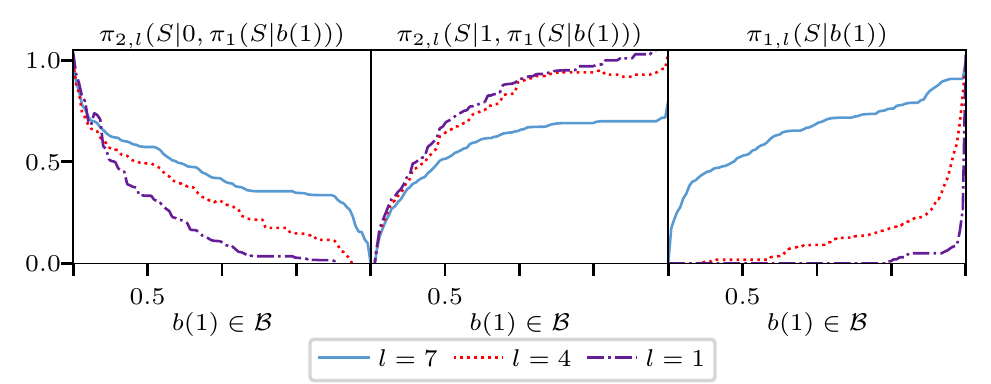}
}
\caption{Probability of the stop action $S$ by the learned equilibrium strategies in function of $b(1)$ and $l$; the left and middle plots show the attacker's stopping probability when $s=0$ and $s=1$, respectively; the right plot shows the defender's stopping probability.}
    \label{fig:stop_prob2}
  \end{figure}

Fig. \ref{fig:exploitability_curve} shows the learning curves of the strategies obtained during the \textsc{T-FP} self-play process. The red curve represents the results from the simulation system and the blue curves show the results from the emulation system. The purple and orange curves give the performance of the Snort IPS baseline and the baseline strategy that mandates a stop action whenever an IPS alert occurs, respectively. The dashed black curve gives the performance of the baseline strategy that assumes knowledge of the exact intrusion time.

The results in Fig. \ref{fig:exploitability_curve} lead us to the following conclusions. First, the fact that all learning curves seem to converge suggests to us that the learned strategies have converged as well. Second, we observe that the exploitability of the learned strategies converges to small values (left plot of Fig. \ref{fig:exploitability_curve}). This indicates that the learned strategies approximate a Nash equilibrium both in the simulation system and in the emulation system. Third, we see from the middle plot in Fig. \ref{fig:exploitability_curve} that both baseline strategies show decreasing performance as the attacker updates its strategy. In contrast, the learned defender strategy improves its performance over time. This shows the benefit of using a game-theoretic approach, whereby the defender's strategy is optimized against a dynamic attacker.

\begin{figure}
\centering
\scalebox{0.82}{
      \includegraphics{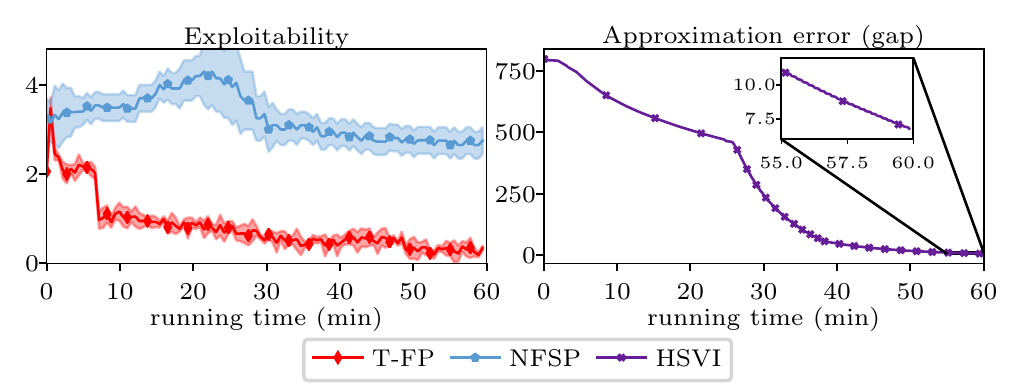}
}
\caption{Comparison between \textsc{T-FP} and two baseline algorithms: NFSP and HSVI; all curves show simulation results; the red curve relate to \textsc{T-FP}; the blue curve to NFSP; the purple curve to HSVI; the left plot shows the approximate exploitability metric and the right plot shows the HSVI approximation error \cite{horak_bosansky_hsvi}; the curves depicting \textsc{T-FP} and NFSP show the mean and the $95\%$ confidence interval over four training runs with different random seeds.}
    \label{fig:converge_times}
  \end{figure}

Fig. \ref{fig:stop_prob2} illustrates some of the structural properties of the learned strategies. The y-axis shows the probability of the stop action $S$ and the x-axis shows the defender's belief $b(1) \in \mathcal{B}$. We observe that the strategies are stochastic. Hence, since Fig. \ref{fig:exploitability_curve} suggests that the learned strategies converge to a Nash equilibrium, Fig. \ref{fig:stop_prob2} suggests that this equilibrium is \textit{mixed}, which we expect based on Theorem \ref{thm:best_responses}.A. As we further expect from Theorem \ref{thm:best_responses}.B-C, we see that that the defender's stopping probability is increasing with $b(1)$ and decreasing with $l$ (right plot of Fig. \ref{fig:stop_prob2}). Similarly, we observe that the attacker's stopping probability is decreasing with the defender's stopping probability when $s=0$ and is increasing when $s=1$ (left and middle plot of Fig. \ref{fig:stop_prob2}).

Fig. \ref{fig:converge_times} allows a comparison between \textsc{T-FP} and the two baseline algorithms (NFSP and HSVI) as they execute in the simulation system. Since \textsc{T-FP} and NFSP both implement fictitious self-play, they allow for a direct comparison. We observe that \textsc{T-FP} converges much faster to a Nash equilibrium than NFSP. We expect the fast convergence of \textsc{T-FP} due to its design to exploit structural properties of the stopping game.

The right plot of Fig. \ref{fig:converge_times} shows that HSVI reaches an HSVI approximation error $<5$ within an hour. We expected slower convergence due to findings in \cite{horak_thesis,horak_solving_one_sided_posgs}. A direct comparison between \textsc{T-FP} and HSVI is not possible due to the different nature of the two algorithms.

Lastly, Fig. \ref{fig:value_fun} shows the computed value function of the game $\hat{V^{*}}$ (Eq. \ref{eq:bellman_posg_1}). We see that $\hat{V^{*}}$ is piece-wise linear and convex, as expected by the theory of one-sided POSGs \cite{horak_thesis}. We also observe that the value of $\hat{V^{*}}$ is minimal when $b(1) \approx 0.25$ and is $0$ when $b(1)=1$. Moreover, we note that this value is slightly lower for $l=1$ compared to $l=7$.
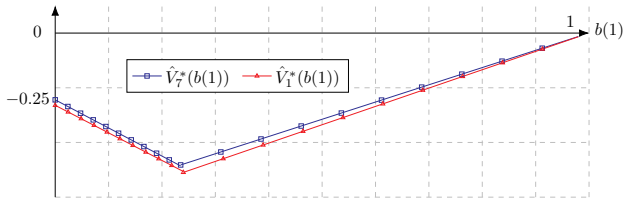
\begin{figure}
  \centering
  \scalebox{0.65}{
    \input{tikz/value_fun.tex}
 }
    \caption{The value function $\hat{V^{*}}_l(b(1))$ computed through the HSVI algorithm with approximation error $4$; the blue and red curves relate to $l=7$ and $l=1$, respectively.}
    \label{fig:value_fun}
\end{figure}
\section{Related Work}\label{sec:related_work}
Traditional approaches to intrusion prevention use packet inspection and static rules for detection of intrusions and selection of response actions \cite{snort,ids_survey, int_prevention}. Their main drawback lies in the need for domain experts to configure the rule sets. As a consequence, much effort has been devoted to developing methods for finding security strategies in an automatic way. This research uses concepts and methods from various areas, most notably from anomaly detection (see example \cite{anomaly_ddetection}), change-point detection (see example \cite{tartakovsky_1}), statistical learning (see examples \cite{fung_ids_distributed,fung_ids_distributed_dirichlet,ml_anomaly_detection}), control theory (see survey \cite{Miehling_control_theoretic_approaches_summary}), game theory (see textbooks \cite{nework_security_alpcan,tambe,carol_book_intrusion_detection,levente_book}), artificial intelligence (see survey \cite{ai_survey} and textbook \cite{al_shaer_book}), dynamic programming (see example \cite{dp_security_1}), reinforcement learning (see surveys \cite{deep_rl_cyber_sec,control_rl_reviews} and textbook \cite{cybenko_acd}), evolutionary methods (see examples \cite{armsrace_malware,hemberg_oreily_evo}), and attack graphs (see example \cite{miehling_attack_graph}). While the research reported in this paper is informed by all the above works, we limit the following discussion to prior work that uses game-theoretic models and centers around finding strategies through reinforcement learning.

\subsection{Reinforcement Learning in Network Security}
Many recent results of automating security strategies have been obtained using reinforcement learning methods. In particular, a large number of studies have focused on intrusion prevention use cases similar to the one we discuss in this paper \cite{hammar_stadler_cnsm_20,hammar_stadler_cnsm_21, elderman, schwartz_2020, oslo_pentest_rl, kurt_rl, microsoft_red_teaming, ridley_ml_defense, rl_cyberdefense_heartbleed, deep_hierarchical_rl_pentest, pentest_rl_rohit, adaptive_cyber_defense_pomdp_rl, muzero_sdn,atmos,sdn_rl_ddos,deep_air,noms_demo_preprint,game_cyber_rl_sim,al_shaer_book_ppo_simulation,cmu_ppo_selfplay}.

These works use a variety of models, including MDPs \cite{oslo_pentest_rl,ridley_ml_defense,deep_hierarchical_rl_pentest,pentest_rl_rohit,deep_air,al_shaer_book_ppo_simulation}, Markov games \cite{elderman, hammar_stadler_cnsm_20, muzero_sdn,game_cyber_rl_sim}, attack graphs \cite{cmu_ppo_selfplay}, and POMDPs \cite{hammar_stadler_cnsm_21,adaptive_cyber_defense_pomdp_rl,kurt_rl}, as well as various reinforcement learning algorithms, including Q-learning \cite{elderman,oslo_pentest_rl,ridley_ml_defense,mec_game_rl_q_learning_sim}, SARSA \cite{kurt_rl}, PPO \cite{hammar_stadler_cnsm_20,hammar_stadler_cnsm_21,cmu_ppo_selfplay,al_shaer_book_ppo_simulation}, hierarchical reinforcement learning \cite{deep_hierarchical_rl_pentest}, DQN \cite{pentest_rl_rohit}, Thompson sampling \cite{adaptive_cyber_defense_pomdp_rl}, MuZero \cite{muzero_sdn}, NFQ \cite{atmos}, DDQN \cite{deep_air}, NFSP \cite{nfsp_jamming_1_sim}, and DDPG \cite{sdn_rl_ddos,game_cyber_rl_sim}.

This paper differs from the works referenced above in three main ways. First, we model the intrusion prevention use case as a partially observed stochastic game. Most of the other works model the use case as a single-agent MDP or POMDP. The advantage of using a game-theoretic model is that it allows finding defender strategies that are effective against dynamic attackers.

Second, in a novel approach, we derive structural properties of strategies in the game using optimal stopping theory.
%This type of analysis is not provided in the referenced works.

Third, our method to find effective defender strategies includes using an emulation system in addition to a simulation system. The advantage of our method compared to the simulation-only approaches \cite{hammar_stadler_cnsm_20,hammar_stadler_cnsm_21, elderman, schwartz_2020, oslo_pentest_rl, kurt_rl, microsoft_red_teaming, ridley_ml_defense, rl_cyberdefense_heartbleed, deep_hierarchical_rl_pentest, pentest_rl_rohit, adaptive_cyber_defense_pomdp_rl,game_cyber_rl_sim,cmu_ppo_selfplay,al_shaer_book_ppo_simulation,mec_game_rl_q_learning_sim,nfsp_jamming_1_sim} is that the parameters of our simulation system are determined by measurements from an emulation system instead of being chosen by a human expert. Further, the learned strategies are evaluated in the emulation system, not in the simulation system. As a consequence, the evaluation results give higher confidence of the obtained strategies' performance in the target infrastructure than what simulation results would provide.

Some prior work on automated intrusion prevention that make use of emulation are: \cite{muzero_sdn}, \cite{atmos}, \cite{sdn_rl_ddos}, \cite{9328143}, and \cite{deep_air}. They emulate software-defined networks based on Mininet \cite{mininet}. The main differences between these efforts and the work described in this paper are: (1) we develop our own emulation system which allows for experiments with a large variety of exploits; (2) we focus on a different intrusion prevention use case; (3) we do not assume that the defender has perfect observability; (4) we do not assume a static attacker; and (5), we use an underlying theoretical framework to formalize the use case, derive structural properties of optimal strategies, and test these properties in an emulation system.

Finally, \cite{cyborg}, \cite{cygil}, and \cite{farland} describe ongoing efforts in building emulation platforms for reinforcement learning and cyber defense, which resemble our emulation system. In contrast to these papers, our emulation system has been built to investigate the specific use case of intrusion prevention and forms an integral part of our general solution method (see Fig. \ref{fig:method}).
\subsection{Game Theoretic Modeling in Network Security}
Several examples of game theoretic security models can be found in the literature, e.g. advanced persistent threat games \cite{flipit,dynamic_game_linan_zhu,general_sum_markov_games_for_strategic_detection_of_apt}, honeypot placement games \cite{honeypot_game,DBLP:journals/compsec/HorakBTKK19,game_theoretic_modeling_ofhoneypot_selection}, resource allocation games \cite{game_resource_alloc_malicious_packet}, authentication games \cite{serkan_gyorgy_game}, distributed denial-of-service games \cite{9328143,posg_cyber_deception_network_epidemic}, and intrusion prevention games \cite{stocahstic_games_security_indep_nodes_nguyen_alpcan_basar, hammar_stadler_cnsm_20, muzero_sdn,optimal_thresholds_for_ids,a_game_theoretic_ids_control_sys_alpcan_basar,zhu_basar_dynamic_policy_ids_config}.

This paper differs from the works referenced above in two main ways. First, we model the intrusion prevention use case as an optimal stopping game. The benefit of our model is that it provides insight into the structure of best response strategies through the theory of optimal stopping.

Game-theoretic formulations based on optimal stopping theory can be found in prior research on Dynkin games \cite{dynkin_orig_3,dynkin_orig_2,dynkin_example_1,dynkin_example_2,dynkin_example_3}. Compared to these papers, our approach is more general by (1) allowing each player to take multiple stop actions within an episode; and (2), by not assuming a game of perfect information. Another difference is that the referenced papers either study purely mathematical problems or problems in mathematical finance. To the best of our knowledge, we are the first to apply the stopping game formulation to the use case of intrusion prevention.

Our stopping game has similarities with the FlipIt game \cite{flipit} and signaling games \cite{signaling_game_original}, both of which are commonplace in the security literature (see survey \cite{game_t_sec_survey} and textbooks \cite{nework_security_alpcan,tambe,carol_book_intrusion_detection,levente_book}). Signaling games have the same information asymmetry as our game and FlipIt uses the same binary state space to model the state of an intrusion. The main differences are as follows. FlipIt models the use case of advanced persistent threats and is a symmetric non-zero-sum game. In contrast, our game models an intrusion prevention use case and is an asymmetric zero-sum game. Compared to signaling games, the main differences are that our game is a sequential and simultaneous-move game. Signaling games are typically two-stage games where one player moves in each stage.

Second, as we noted above, we evaluate obtained strategies on an emulated IT infrastructure. This contrasts with most of the prior works that use game-theoretic approaches, which only evaluate strategies analytically or in simulation \cite{flipit,dynamic_game_linan_zhu,general_sum_markov_games_for_strategic_detection_of_apt,honeypot_game,DBLP:journals/compsec/HorakBTKK19,game_theoretic_modeling_ofhoneypot_selection,serkan_gyorgy_game,posg_cyber_deception_network_epidemic,stocahstic_games_security_indep_nodes_nguyen_alpcan_basar, hammar_stadler_cnsm_20,optimal_thresholds_for_ids,a_game_theoretic_ids_control_sys_alpcan_basar,zhu_basar_dynamic_policy_ids_config}.

%Some prior game-theoretic approaches to intrusion prevention that make use of emulation are: \cite{muzero_sdn}, \cite{atmos}, \cite{sdn_rl_ddos}, and \cite{deep_air}. They emulate software-defined networks based on Mininet \cite{mininet}.

%The game-theoretic approaches proposed in prior work that most resemble ours are \cite{muzero_sdn} and \cite{9328143}. They both propose partially observed game models and evaluate strategies using emulations based on Mininet \cite{mininet}. This paper differs from the referenced works in two main ways.
%First, we formulate the intrusion prevention use case as an optimal stopping game. The advantage of our approach is that we obtain structural properties of best response strategies through the theory of optimal stopping. This contrasts with \cite{muzero_sdn} that uses a general game model and does not obtain any theoretical insight into the structure of best response strategies or equilibria. The work in \cite{9328143} obtains similar structural properties but studies a different use case and uses a different game model.
%Second, while both \cite{muzero_sdn} and \cite{9328143} evaluate strategies using emulations based on Mininet, we develop our own emulation system which allows for experiments with a large variety of exploits and defenses. %
%Using intrusion prevention as an example,
\section{Conclusion and Future Work}\label{sec:conclusions}
We formulate the interaction between an attacker and a defender in an intrusion prevention use case as an optimal stopping game. The theory of optimal stopping provides us with insight about optimal strategies for attackers and defenders. Based on this knowledge, we develop a fictitious self-play algorithm, \textsc{T-FP}, which allows us to compute near optimal strategies in an efficient way. This approach provides us with a complete formal framework for analyzing and solving the intrusion prevention use case. The simulation results from executions of \textsc{T-FP} show that the exploitability of the computed strategies converges, which suggests that the strategies converge to a Nash equilibrium and thus to an optimum in the game-theoretic sense. The results also demonstrate that \textsc{T-FP} converges faster than a state-of-the-art fictitious self-play algorithm by taking advantage of structural properties of optimal stopping strategies.

To assess the computed strategies in a real environment, we evaluate them in a system that emulates our target infrastructure. The results show that the strategies achieve almost the same performance in the emulated infrastructure as in the simulation. This gives us a high confidence of the obtained strategies' performance in the target infrastructure.

We plan to extend this work in several directions. First of all, the model of the attacker and the defender in this paper is simplistic as it only models the timing of actions and not their selection. We plan to combine our current model for deciding when to take defensive actions with a model for the selection of which action to take.
\section{Acknowledgments}
This research has been supported in part by the Swedish armed forces and was conducted at KTH Center for Cyber Defense and Information Security (CDIS). The authors would like to thank Pontus Johnson for his useful input to this research, and Forough Shahab Samani and Xiaoxuan Wang for their constructive comments on a draft of this paper. The authors are also grateful to Branislav Bosansk{\'{y}} for sharing the code of the HSVI algorithm for one-sided POSGs and to Jakob Stymne for contributing to our implementation of NFSP.
\appendices
\section{Proof of Theorem \ref{thm:best_responses}.A}\label{appendix:theorem_1_a}
\begin{proof}[Proof of Theorem \ref{thm:best_responses}.A.]
Since the POSG $\Gamma$ introduced in Section \ref{sec:formal_model_2} is finite and $\gamma \in (0,1)$, the existence proofs in \cite{posg_equilibria_existence_finite_horizon} and \cite{horak_thesis} applies, which state that a mixed Nash equilibrium exists.

We prove that a pure Nash equilibrium exists when $s=0 \iff b(1) = 0$ using a proof by construction. It follows from Eqs. \ref{eq:reward_0}-\ref{eq:reward_5} and Eq. \ref{eq:br_defender} that the pure strategy defined by $\bar{\pi}_{1,l}(0)=C$ and $\bar{\pi}_{1,l}(b(1))=S \iff b(1)>0$ is a best response for the defender against any attacker strategy when $s=0 \iff b(1) = 0$. Similarly, given $\bar{\pi}_{1,l}$, we get from Eqs. \ref{eq:reward_0}-\ref{eq:reward_5} and Eq. \ref{eq:br_attacker} that the pure strategy defined by $\bar{\pi}_{2,l}(0,b(1))=C$ and $\bar{\pi}_{2,l}(1,b(1))=S$ for all $b(1) \in [0,1]$ is a best response for the attacker. Hence, $(\bar{\pi}_{1,l},\bar{\pi}_{2,l})$ is a pure Nash equilibrium (see Eq. \ref{eq:minmax_objective}).
\end{proof}
\section{Proof of Theorem \ref{thm:best_responses}.B.}\label{appendix:theorem_1_b}
\begin{proof}[]
Given the POSG $\Gamma$ introduced in Section \ref{sec:formal_model_2} and a fixed attacker strategy $\pi_{2,l}$, the best response strategy of the defender $\tilde{\pi}_{1,l}\in B_1(\pi_{2,l})$ is an optimal strategy in a POMDP $\mathcal{M}^{P}$ (see Section \ref{sec:game_analysis}). Hence, it is sufficient to show that there exists an optimal strategy $\pi_{1,l}^{*}$ in $\mathcal{M}^{P}$ that satisfies Eq. \ref{eq:prop_br_defender}. The conditions for Eq. \ref{eq:prop_br_defender} to hold and the proof are given in our previous work \cite{hammar_stadler_tnsm}[Theorem 1.C]. Since $f_{O|s}$ is TP2 by assumption and all of the remaining conditions hold by definition of $\Gamma$, the result follows.
\end{proof}
\section{Proof of Theorem \ref{thm:best_responses}.C.}\label{appendix:theorem_1_c}
Given the POSG $\Gamma$ introduced in Section \ref{sec:formal_model_2} and a fixed defender strategy $\pi_{1,l}$, the best response strategy of the attacker $\tilde{\pi}_{2,l} \in B_2(\pi_{1,l})$ is an optimal strategy in an MDP $\mathcal{M}$ (see Section \ref{sec:game_analysis}). Hence, it is sufficient to show that there exists an optimal strategy $\pi_{2,l}^{*}$ in $\mathcal{M}$ that satisfies Eqs. \ref{eq:prop_br_attacker_1}-\ref{eq:prop_br_attacker_1}. To prove this, we use properties of $\mathcal{M}$'s value function $V_{\pi_{1,l},l}^{*}$.

We use the value iteration algorithm to establish properties of $V_{\pi_{1,l},l}^{*}$ \cite{puterman,krishnamurthy_2016}. Let $V_{\pi_{1,l},l}^{k}$, $\mathscr{S}^{k,(2)}_{s,l,\pi_{1,l}}$, and $\mathscr{C}^{k,(2)}_{s,l,\pi_{1,l}}$, denote the value function, the stopping set, and the continuation set at iteration $k$ of the value iteration algorithm, respectively. Then, $\lim_{k\rightarrow \infty}V_{\pi_{1,l},l}^k=V_{\pi_{1,l},l}^{*}$, $\lim_{k\rightarrow \infty}\mathscr{S}^{k,(2)}_{s,l,\pi_{1,l}}=\mathscr{S}^{(2)}_{s,l,\pi_{1,l}}$, and $\lim_{k\rightarrow \infty}\mathscr{C}^{k,(2)}_{s,l,\pi_{1,l}}$ $=\mathscr{C}^{(2)}_{s,l,\pi_{1,l}}$ \cite{puterman,krishnamurthy_2016}. We define $V_{\pi_{1,l},l}^0\big($$(s,b(1))$$\big)=0$ for all $b(1)\in [0,1]$, $s \in \mathcal{S}$ and $l\in \{1,\hdots,L\}$.

Towards the proof of Theorem \ref{thm:best_responses}.C, we state the following six lemmas.
\begin{lemma}\label{lemma:V_geq_zero}
Given any defender strategy $\pi_{1,l}$, $V^{*}_{\pi_{1,l,2}}\big($$s,b(1)\big)$ $\geq 0$ for all $s \in \mathcal{S}$ and $b(1) \in [0,1]$.
\end{lemma}
\begin{proof}[Proof.]
Consider $\bar{\pi}_{2,l}$ defined by $\bar{\pi}_{2,l}(0,\cdot)=C$ and $\bar{\pi}_{2,l}(1,\cdot)=S$. Then it follows from Eqs. \ref{eq:reward_0}-\ref{eq:reward_5} that for any $\pi_{1,l} \in \Pi_1$, $s \in \mathcal{S}$ and $b(1) \in [0,1]$, the following holds: $V^{\bar{\pi}_{2,l}}_{\pi_{1,l},l}(s,b(1)) \geq 0$. By optimality, $V^{\bar{\pi}_{2,l}}_{\pi_{1,l},l}(s,b(1)) \leq V^{*}_{\pi_{1,l},l}(s,b(1))$. Hence, $V^{*}_{\pi_{1,l},l}(s,b(1)) \geq 0$.
\end{proof}

\begin{lemma}\label{lemma:decreasing_v}
$V^{*}_{\pi_{1,l},l}\big($$1,b(1)\big)$ is non-increasing with $\pi_{1,l}(S|b(1))$ and non-decreasing with $l \in \{1,\hdots,L\}$.
\end{lemma}
\begin{proof}[Proof.]
  We prove this statement by mathematical induction. For $k=0$, we know from Eqs. \ref{eq:reward_0}-\ref{eq:reward_5} that $V^0_{\pi_{1,l},l}\big($$1$, $b(1)\big)$ is non-increasing with $\pi_{1,l}(S|b(1))$ and non-decreasing with $l$.

For $k > 0$, $V^{k}_{\pi_{1,l},l}$ is given by:
\begin{align}
  &V^{k}_{\pi_{1,l},l}\big(1,b(1)\big) = \max\Big[0, -R\big(1,(C,a^{(1)})\big) \label{eq:decreasing_v}\\
  &+ (1-\phi_{l})\sum_of_O(o|1)V^{k-1}_{l-a^{a(1)}}\big(1,b(1)\big)\Big]\nonumber
\end{align}
The first term in the maximization in Eq. \ref{eq:decreasing_v} is trivially non-increasing with $\pi_{1,l}(S|b(1))$ and non-decreasing with $l$. Assume by induction that the conditions hold for $V^{k-1}_{\pi_{1,l},l}\big($$s$, $b(1)\big)$. Then the second term in Eq. \ref{eq:decreasing_v} is non-increasing with $\pi_{1,l}(S|b(1))$ and non-decreasing with $l$ by Eqs. \ref{eq:reward_0}-\ref{eq:reward_5} and the induction hypothesis. Hence, $V^{k}_{\pi_{1,l},l}\big($$s$, $b(1)\big)$ is non-increasing with $\pi_{1,l}(S|b(1))$ and non-decreasing with $l$ for all $k\geq 0$.
\end{proof}

\begin{lemma}\label{lemma:dec_stopping}
If $f_{O}$ is TP2 and $\pi_{1,l}(S|b(1))$ is increasing with $b(1)$, then $V_{\pi_{1,l},l}(b(1),1) \geq \sum_{o}f_O(o|1)V_{\pi_{1,l},l}(1,b^{o}(1))$, where $b^{o}(1)$ denotes $b(1)$ updated with Eq. \ref{eq:belief_upd} after observing $o \in \mathcal{O}$.
\end{lemma}
\begin{proof}[Proof.]
  Since $f_O$ is TP2, it follows from \cite[Theorem 10.3.1, pp. 225,238]{krishnamurthy_2016} and \cite[Lemma 4, pp. 12]{hammar_stadler_tnsm} that given two beliefs $b^{\prime}(1) \geq b(1)$ and two observations $o \geq \bar{o}$, the following holds for any $k \in \mathcal{O}$ and $l_t \in \{1,\hdots, L\}$: $b^{\prime,o}(1) \geq b^{o}(1)$, $\mathbb{P}[o \geq k |b^{\prime}(1)] \geq \mathbb{P}[o \geq k |b(1)]$, and $b_{a}^{o}(1) \geq b_{a}^{\bar{o}}(1)$.

Since $\pi_{1,l}$ is increasing with $b(1)$ and $V_{\pi_{1,l},l}(b(1),1)$ is decreasing with $b(1)$ (Lemma \ref{lemma:decreasing_v}), it follows that $\mathbb{E}_{o}[b^{o}(1)] \geq b(1)$, and thus $V_{\pi_{1,l},l}(b(1),1) \geq \sum_{o}f_O(o|1)V_{\pi_{1,l},l}(1,b^{o}(1))$.
\end{proof}

\begin{lemma}\label{lemma:stop_if_defender_knows}
If $f_{O}$ is TP2, $\pi_{1,l}(S|b(1))=1$, and $\pi_{1,l}(S|b(1))$ is increasing with $b(1)$, then $V^{*}_{\pi_{1,l,2}}\big($$s,$$b(1)\big)$$=0$ and for any $\tilde{\pi}_{2,l} \in B_2(\pi_{1,l})$, $\tilde{\pi}_{2,l}$$(1,$$b(1))$$=1$.
\end{lemma}
\begin{proof}[Proof.]
From Eqs. \ref{eq:bellman_eq_41}-\ref{eq:bellman_posg_1} we know that $\tilde{\pi}_{2,l}(1,b(1))=1$ iff:
\begin{align}
R_{st}/l + (\phi_l-1)\sum_of_O(o|1)V^{*}_{\pi_{1,l},l-a^{(1)}}(1, b^{o}(1)) \geq 0 \label{eq:ineq_stop_if_def}
\end{align}
We know that $R_{st} \geq 0$ (see Section \ref{sec:formal_model_2}). Further, since $f_{O}$ is TP2, $\pi_{1,l}(S|b(1))=1$, and $\pi_{1,l}(S|b(1))$ is increasing with $b(1)$, we have by Lemma \ref{lemma:dec_stopping} that $\mathbb{E}_{o}[\pi_{1,l}(S|b^{o}(1))=1]$. The second term in the left-hand side of Eq. \ref{eq:ineq_stop_if_def} is thus zero. Hence, the inequality holds and $\tilde{\pi}_{2,l}(1,b(1))=1$, which implies that $V^{*}_{\pi_{1,l,2}}\big($$s,b(1)\big)=0$.
\end{proof}

\begin{lemma}\label{lemma:continue_stop_equality}
Given any defender strategy $\pi_{1,l}\in \Pi_1$, if $\pi^{*}_{2,l}(1,b(1)) = S$, then $\pi^{*}_{2,l}(0,b(1)) = C$.
\end{lemma}
\begin{proof}[Proof.]
$\pi^{*}_{2,l}$$(1,$$b(1))$$=S$ implies that $V^{*}_{\pi_{1,l},l}$$(1,$$b(1))$$=0$. Hence, by Lemma \ref{lemma:dec_stopping} we get that:
\begin{align}
 &(1-\phi_l)\sum_{o\in \mathcal{O}} f_O(o|1)V^{*}_{\pi_{1,l},l}(1,b^{o}) \leq 0\\
&\implies \sum_{o\in \mathcal{O}} f_O(o|1)V^{*}_{\pi_{1,l},l}(1,b^{o}) \leq \sum_{o\in \mathcal{O}} f_O(o|0)V^{*}_{\pi_{1,l},l}(0,b^{o})\nonumber\\
&\implies \pi^{*}_{2,l}(0,b(1)) = C\nonumber
\end{align}
\end{proof}

\begin{lemma}\label{lemma:decreasing_differences}
If $\pi_{1,l}(S|b(1))$ is non-decreasing with $b(1)$ and $f_O$ is TP2, then $V^{*}_{\pi_{1,l},l}\big($ $0,b(1)\big) - V^{*}_{\pi_{1,l},l}\big($$1,b(1)\big)$ is non-decreasing with $\pi_{1,l}(S|b(1))$.
\end{lemma}
\begin{proof}[Proof.]
We prove this statement by mathematical induction. Let $W^{k}_{\pi_{1,l},l}(b(1))$ $= V^{k}_{\pi_{1,l},l}\big($$0,b(1)\big) - V^{k}_{\pi_{1,l},l}\big($$1,b(1)\big)$. For $k=0$, it follows from Eqs. \ref{eq:reward_0}-\ref{eq:reward_5} that $W^{0}_{\pi_{1,l},l}(b(1))$ is non-decreasing with $\pi_{1,l}(S|b(1))\in [0,1]$. Assume by induction that the conditions hold for $W^{k-1}_{\pi_{1,l},l}(b(1))$. We show that then the conditions hold also for $W^{k}_{\pi_{1,l},l}(b(1))$.

There are three cases to consider:
\begin{itemize}
\item If $b(1) \in \mathscr{S}^{k,(2)}_{0,l,\pi_{1,l}} \cap \mathscr{C}^{k,(2)}_{1,l,\pi_{1,l}}$, then:
\begin{align}
  &W^{k}_{\pi_{1,l},l}(b(1)) = R_{int} + \\
  &\pi_{1,l}(S|b(1))(R_{st}/l -R_{cost}/l - R_{int})\nonumber
\end{align}
which is non-decreasing with $\pi_{1,l}(S|b(1))$ since $R_{st}/l -R_{cost}/l- R_{int}\geq 0$ (see Section \ref{sec:formal_model_2}).
\item If $b(1) \in \mathscr{C}^{k,(2)}_{0,l,\pi_{1,l}} \cap \mathscr{C}^{k,(2)}_{1,l,\pi_{1,l}}$, then:
\begin{align}
  &W^{k}_{\pi_{1,l},l}(b(1)) = \pi_{1,l}(S|b(1))\Big(R_{st}/l -R_{cost}/l\\
  &-R_{int}\Big) + V^{k-1}_{\pi_{1,l},l}\big(1,b(1)\big)\Big) + R_{int} +\sum_of_{O}(o|0)\nonumber\\
  &V^{k}_{\pi_{1,l},l}\big(0,b^{o}(1)\big) - (1-\phi_l)f_{O}(o|1)V^{k}_{\pi_{1,l},l}\big(1,b^{o}(1)\big)\nonumber
\end{align}
The first term is non-decreasing with $\pi_{1,l}(S|b(1))$ since $R_{st}/l -R_{cost}/l- R_{int}\geq 0$ (see Section \ref{sec:formal_model_2}) and $V^{k-1}_{\pi_{1,l},l}\big(1,b(1)\big) \geq 0$ (it is a consequence of Lemma \ref{lemma:V_geq_zero} and Eqs. \ref{eq:bellman_eq_41}-\ref{eq:bellman_posg_1}). The second term is non-decreasing with $\pi_{1,l}(S|b(1))$ by the induction hypothesis and the assumption that $f_O$ is TP2.
\item If $b(1) \in \mathscr{C}^{k,(2)}_{0,l,\pi_{1,l}} \cap \mathscr{S}^{k,(2)}_{1,l,\pi_{1,l}}$, then:
\begin{align}
  &W^{k}_{\pi_{1,l},l}(b(1)) = \pi_{1,l}(S|b(1))(-R_{cost}/l) \label{eq:step_1_1}\\
  & + \sum_of_{O}(o|0)V^{k}_{\pi_{1,l},l}\big(0,b^{o}(1)\big) \nonumber\\
  &=\pi_{1,l}(S|b(1))(-R_{cost}/l) + \sum_of_{O}(o|0)\cdot \nonumber\\
  &V^{k}_{\pi_{1,l},l}\big(0,b^{o}(1)\big) - (1-\phi_l)f_{O}(o|1)V^{k}_{\pi_{1,l},l}\big(1,b^{o}(1)\big)\nonumber
\end{align}
The first term is non-decreasing with $\pi_{1,l}(S|b(1))$ since $-R_{cost}/l\geq 0$. The second term is non-decreasing with $\pi_{1,l}(S|b(1))$ by the induction hypothesis and the assumption that $f_O$ is TP2. The second equality in Eq. \ref{eq:step_1_1} follows from Lemma \ref{lemma:dec_stopping} and because $b(1) \in \mathscr{S}^{k,(2)}_{1,l,\pi_{1,l}}$.
\end{itemize}
The other cases, e.g. $b(1) \in \mathscr{S}^{k,(2)}_{0,l,\pi_{1,l}} \cap \mathscr{S}^{k,(2)}_{1,l,\pi_{1,l}}$, can be discarded due to Lemma \ref{lemma:continue_stop_equality}. Hence, $W^{k}_{\pi_{1,l},l}(b(1))$ is non-decreasing with $\pi_{1,l}(S|b(1))$ for all $k \geq 0$.
\end{proof}
We now use Lemmas \ref{lemma:V_geq_zero}-\ref{lemma:decreasing_differences} to prove Theorem \ref{thm:best_responses}.C. The main idea behind the proof is to show that the stopping sets in state $s=1$ have the form: $\mathscr{S}^{2}_{1,l,\pi_{1,l}} = [\tilde{\beta}_{1,l},1]$, and that the continuation sets in state $s=0$ have the form: $\mathscr{C}^{2}_{0,l,\pi_{1,l}} = [\tilde{\beta}_{0,l},1]$, for some values $\tilde{\beta}_{0,1}, \tilde{\beta}_{1,1}, \hdots, \tilde{\beta}_{0,L}, \tilde{\beta}_{1,L} \in [0,1]$.
\begin{proof}[Proof of Theorem \ref{thm:best_responses}.C.]
We first show that $1 \in \mathscr{S}^{(2)}_{1,l,\pi_{1,l}}$ and that $1 \in \mathscr{C}^{(2)}_{0,l,\pi_{1,l}}$. Since $\pi_{1,l}(S|1)=1$, by Lemma \ref{lemma:stop_if_defender_knows} we have that $1 \in \mathscr{S}^{(2)}_{1,l,\pi_{1,l}}$ and it follows from Eqs. \ref{eq:bellman_eq_41}-\ref{eq:bellman_posg_1} that $\tilde{\pi}_{2,l}(0,b(1))=C$ iff:
\begin{align}
  &\sum_of_O(o|0) V^{*}_{\pi_{1,l},l-1}(0, b^{o}(1)) -\nonumber\\
  &f_O(o|1)V^{*}_{\pi_{1,l},l-1}(1, b^{o}(1))\geq 0
\end{align}
The left-hand side of the above equation is positive due to the assumption that $f_{O}$ is TP2 and since $\sum_of_O(o|0)V^{*}_{\pi_{1,l},l-1}(0,b^{o}(1))\geq 0$ by Lemma \ref{lemma:V_geq_zero} and $f_O(o|1)V^{*}_{\pi_{1,l},l-1}(1,$$b^{o}(1))=0$ by Lemma \ref{lemma:dec_stopping}. Hence, $1 \in \mathscr{C}^{(2)}_{0,l,\pi_{1,l}}$.

Now we show that $\mathscr{S}^{2}_{1,l,\pi_{1,l}} = [\tilde{\beta}_{1,l},1]$ and that $\mathscr{C}^{2}_{0,l,\pi_{1,l}} = [\tilde{\beta}_{0,l},1]$ for some values $\tilde{\beta}_{0,1}, \tilde{\beta}_{1,1}, \hdots, \tilde{\beta}_{0,L}, \tilde{\beta}_{1,L} \in [0,1]$. From Eqs. \ref{eq:bellman_eq_41}-\ref{eq:bellman_posg_1} we know that $\tilde{\pi}_{2,l}(1,b(1))=S$ iff:
\begin{align}
&\mathbb{E}_{\pi_{1,l}}\Big[\mathcal{R}_{l_t}\big(1,(a^{(1)}, C)\big)\\
  &(\phi_{l_t}-1)\sum_of_O(o|1)V^{*}_{\pi_{1,l},l-a^{(1)}}(1, b^{o}(1))\Big] \geq 0\nonumber
\end{align}
The first term in the above expectation is increasing with $b(1)$ (Eqs. \ref{eq:reward_0}-\ref{eq:reward_5}). The second term is decreasing with $b(1)$ (Lemma \ref{lemma:decreasing_v}). Hence, we conclude that if $\tilde{\pi}_{2,l}(1,b(1))=S$, then for any $b^{\prime}(1) \geq b(1)$, $\tilde{\pi}_{2,l}(1,b^{\prime}(1))=S$. As a consequence, there exists values $\tilde{\beta}_{1,1}, \hdots, \tilde{\beta}_{1,L}$ such that $\mathscr{S}^{2}_{1,l,\pi_{1,l}} = [\tilde{\beta}_{1,l},1]$.

Similarly, from Eqs. \ref{eq:bellman_eq_41}-\ref{eq:bellman_posg_1} we know that $\tilde{\pi}_{2,l}(0,b(1))=C$ iff:
\begin{align}
&\mathbb{E}_{\pi_{1,l}}\Big[\sum_of_O(o|0)V^{*}_{\pi_{1,l},l-a^{(1)}}(0, b^{o}(1))\label{eq:theorem_b_12}\\
&-f_O(o|1)V^{*}_{\pi_{1,l},l-a^{(1)}}(1, b^{o}(1))\Big] \geq 0\nonumber
\end{align}
Since $f_{O}$ is TP2 and $\pi_{1,l}(S|b(1))$ is increasing with $b(1)$, the left-hand side in the above inequality is decreasing (Lemma \ref{lemma:decreasing_v} and Lemma \ref{lemma:decreasing_differences}). Hence, we conclude that if $\tilde{\pi}_{2,l}(0,b(1))=C$, then for any $b^{\prime}(1) \geq b(1)$, $\tilde{\pi}_{2,l}(0,b^{\prime}(1))=C$. As a result, there exists values $\tilde{\beta}_{0,1}, \hdots, \tilde{\beta}_{0,L}$ such that $\mathscr{C}^{2}_{0,l,\pi_{1,l}} = [\tilde{\beta}_{0,l},1]$.
\end{proof}
\section{Hyperparameters}\label{appendix:hyperparameters}
\begin{table}
\centering
\resizebox{1\columnwidth}{!}{%
  \begin{tabular}{ll} \toprule
  \textbf{Game Parameters} & {\textbf{Values}} \\
  \hline
  $R_{st}, R_{cost}, R_{int}$,$\gamma$, $\phi_{l_t}$, $L$ & $20$, $-2$, $-1$, $0.99$, $1/2l_t$, $7$\\
  {\textbf{\textsc{T-FP} Parameters}} & {\textbf{Values}} \\
  \hline
  $c, \epsilon, \lambda, A, a, N$ & $10$, $0.101$, $0.602$, $100$, $1$, $50$\\
  {\textbf{NFSP Parameters}} & {\textbf{Values}} \\
  \hline
  lr RL, lr SL, batch, \# layers & $10^{-2}$,$5\cdot 10^{-3}$, $64$, $2$ \\
  \# neurons, $\mathcal{M}_{RL}$, $\mathcal{M}_{SL}$ & $128$, $2\times 10^{5}$, $2\times 10^{6}$,\\
  $\epsilon$, $\epsilon$-decay, $\eta$ & $0.06$, $0.001$, $0.1$\\
  {\textbf{HSVI Parameter}} & {\textbf{Value}} \\
  \hline
  $\epsilon$ & $3$\\
  \bottomrule\\
\end{tabular}
}
\caption{Hyperparameters of the POSG and the algorithms used for evaluation.}\label{tab:hyperparams}
\end{table}
The hyperparameters used for the evaluation are listed in Table \ref{tab:hyperparams} and were obtained through grid search.

\section{Configuration of the Infrastructure in Fig. \ref{fig:system2}}\label{appendix:infrastructure_configuration}
The configuration of the target infrastructure (Fig. \ref{fig:system2}) is available in Table \ref{tab:emulation_setup}.
\begin{table}
\centering
\resizebox{1\columnwidth}{!}{%
\begin{tabular}{ll} \toprule
  {\textit{ID (s)}} & {\textit{OS:Services:Exploitable Vulnerabilities}} \\ \midrule
  $N_1$ & Ubuntu20:Snort(community ruleset v2.9.17.1),SSH:- \\
  $N_2$ & Ubuntu20:SSH,HTTP Erl-Pengine,DNS:SSH-pw\\
  $N_4$ & Ubuntu20:HTTP Flask,Telnet,SSH:Telnet-pw \\
  $N_{10}$ &Ubuntu20:FTP,MongoDB,SMTP,Tomcat,TS3,SSH:FTP-pw \\
  $N_{12}$ & Jessie:TS3,Tomcat,SSH:CVE-2010-0426,SSH-pw \\
  $N_{17}$ & Wheezy:Apache2,SNMP,SSH:CVE-2014-6271 \\
  $N_{18}$ & Deb9.2:IRC,Apache2,SSH:SQL Injection \\
  $N_{22}$ & Jessie:PROFTPD,SSH,Apache2,SNMP:CVE-2015-3306 \\
  $N_{23}$ & Jessie:Apache2,SMTP,SSH:CVE-2016-10033 \\
  $N_{24}$ & Jessie:SSH:CVE-2015-5602,SSH-pw \\
  $N_{25}$ & Jessie: Elasticsearch,Apache2,SSH,SNMP:CVE-2015-1427\\
  $N_{27}$ & Jessie:Samba,NTP,SSH:CVE-2017-7494\\
  $N_3$,$N_{11}$,$N_{5}$-$N_9$& Ubuntu20:SSH,SNMP,PostgreSQL,NTP:-\\
  $N_{13-16}$,$N_{19-21}$,$N_{26}$,$N_{28-31}$& Ubuntu20:NTP, IRC, SNMP, SSH, PostgreSQL:-\\
  \bottomrule\\
\end{tabular}
}
\caption{Configuration of the target infrastructure (Fig. \ref{fig:system2}).}\label{tab:emulation_setup}
\end{table}

\ifCLASSOPTIONcaptionsoff
  \newpage
\fi

\bibliographystyle{IEEEtran}
\bibliography{references,url}

\end{document}

%% file: tikz/system5.tex
      \begin{tikzpicture}[fill=white, >=stealth,
    node distance=3cm,
    database/.style={
      cylinder,
      cylinder uses custom fill,
      %%cylinder body fill=yellow!50,
      %%cylinder end fill=yellow!50,
      shape border rotate=90,
      aspect=0.25,
      draw}]

    \tikzset{
node distance = 9em and 4em,
sloped,
   box/.style = {%
    shape=rectangle,
    rounded corners,
    draw=blue!40,
    fill=blue!15,
    align=center,
    font=\fontsize{12}{12}\selectfont},
 arrow/.style = {%
    %%draw=blue!30,
    line width=0.1mm,% <-- select desired width
    -{Triangle[length=5mm,width=2mm]},
    shorten >=1mm, shorten <=1mm,
    font=\fontsize{8}{8}\selectfont},
}

%%  \draw[->, color=black] (8.25, -1.5) to (9, -1.5) to (9, 0.75);
%%  \draw[->, color=black] (3, 0.75) to (3, -1.5) to (3.7, -1.5);

\node[scale=0.8] (kth_cr) at (6,1.6)
{
  \begin{tikzpicture}

\node[inner sep=0pt,align=center, scale=0.8] (time) at (0.1,-5.5)
{
  \begin{tikzpicture}

\node[inner sep=0pt,align=center] (gpu1) at (-2.7,1.75)
  {\scalebox{0.15}{
     \includegraphics{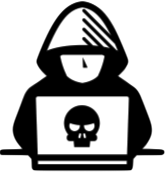}
   }
 };

  \node[inner sep=0pt,align=center, scale=1.2, color=black] (hacker) at (-2.7,2.4)
  {Attacker
  };

\node[inner sep=0pt,align=center] (gpu1) at (-0.30,1.7)
  {\scalebox{0.08}{
     \includegraphics{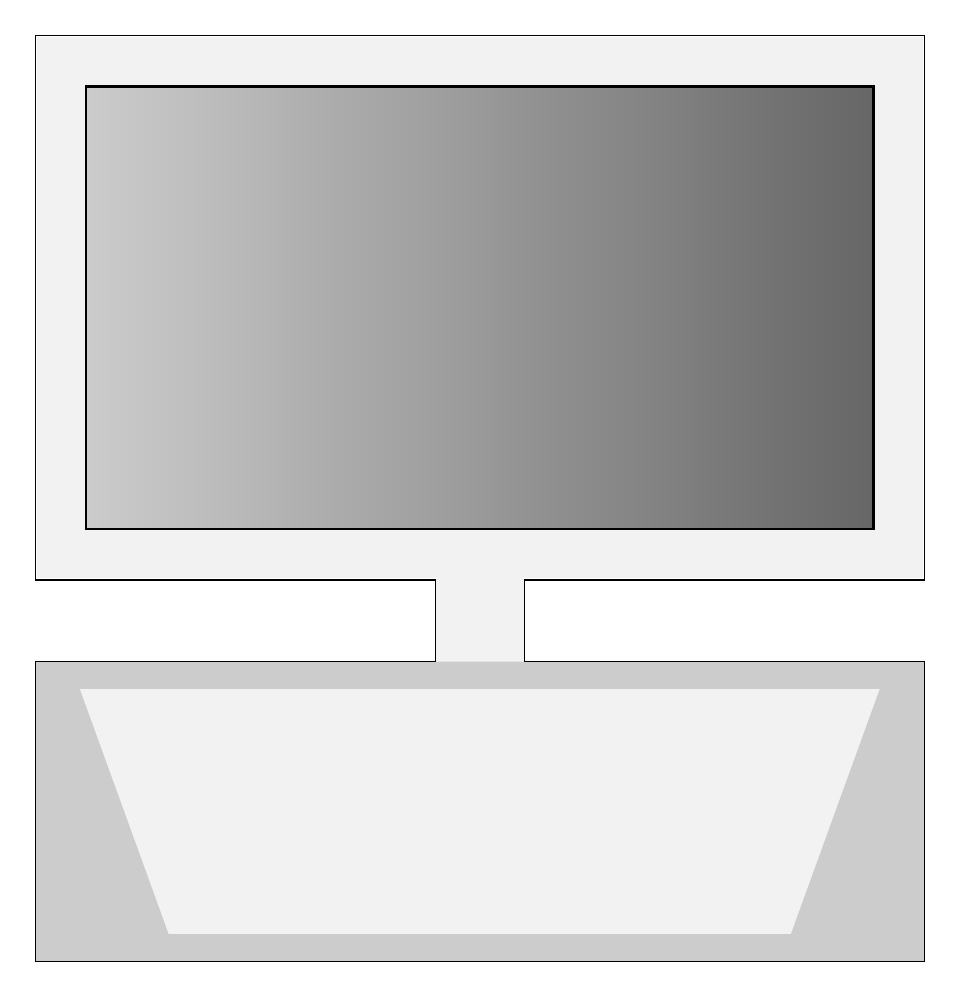}
   }
 };

%%\node[inner sep=0pt,align=center] (gpu1) at (-0.2,1.7)
%%  {\scalebox{0.8}{
%%     \includegraphics{workstation.eps}
%%   }
%% };

%%\node[inner sep=0pt,align=center] (gpu1) at (1.2,1.7)
%%  {\scalebox{0.08}{
%%     \includegraphics{laptop3.pdf}
%%   }
%% };
\node[inner sep=0pt,align=center] (c2) at (1.7,1.7)
  {\scalebox{0.08}{
     \includegraphics{laptop3.pdf}
   }
 };

  \node[inner sep=0pt,align=center, scale=1.2] (hacker) at (0.65,2.4)
  {Clients
  };

  \node[inner sep=0pt,align=center, scale=1.2] (hacker) at (0.75,1.65)
  { \LARGE$\hdots$
  };

%%  \node[inner sep=0pt,align=center, scale=1.2] (hacker) at (2.65,2.4)
%%  {Client
%%  };

  \draw[-, color=black, fill=black!5] (-4.3,-7) to (3,-7) to (3,0.75) to (-4.3, 0.75) to (-4.3,-7);

  \draw[-, color=black] (-0.4,1.1) to (-0.4,1.33);
%%  \draw[-, color=black] (1.1,1.1) to (1.1,1.33);
  \draw[-, color=black] (-2.7,1.1) to (-2.7,1.33);

  \draw[-, color=black] (1.6,1.1) to (1.6,1.33);

  \draw[-, color=black] (-4.4,1.1) to (3.1,1.1);

\node[inner sep=0pt,align=center] (gpu1) at (-0.65,-8)
  {\scalebox{0.08}{
     \includegraphics{laptop3.pdf}
   }
 };

 \draw[-, color=black] (-0.75,-6) to (-0.75,-7.3);
 \draw[-, color=black] (-0.75,-6) to (0,-6);
 \draw[-, color=black] (-0.75,-6) to (-0.75,-5.5);

 \draw[-, color=black] (-4.4,-7.3) to (3.1,-7.3);

 \draw[-, color=black] (-0.75,-7.3) to (-0.75,-7.6);

   \node[inner sep=0pt,align=center, scale=1.2] (hacker) at (-0.68,-8.6)
  {\textcolor{black}{Defender}
  };
    \node[scale=0.15](R1) at (-0.75,-0.4) {\router{}};

    \node[inner sep=0pt,align=center, scale=0.5] (hacker) at (-0.75,-0.52)
 {
$1$
};

  \node[inner sep=0pt,align=center, scale=1] (hacker) at (1.1,-0.4)
  {
    IPS
 };
 \node[server, scale=0.8, fill=Blue, color=Blue](s11) at (0.5,-0.4) {};
    \node[inner sep=0pt,align=center, scale=0.5] (hacker) at (0.25,-0.4)
 {
$1$
};

 \draw[-, color=black] (-0.75,-0.2) to (-0.75,1.1);
 \draw[->, color=black, dashed] (0.35,0.05) to (0.35,0.35) to (-0.75, 0.35);

  \node[inner sep=0pt,align=center, scale=1] (hacker) at (-0.1,0.55)
  {
    \small \textit{alerts}
  };

\node[inner sep=0pt,align=center, scale=1] (hacker) at (-1.5,-0.0)
  {
Gateway
};

%%\node[database, minimum width=0.5cm, minimum height=0.8cm,align=center] (s6) at (-3,-0.4) {$\quad$};
%% \draw[<->, color=black, dashed] (-2.7,-0.4) to (-1.2,-0.4);

 \node[server, scale=0.8](s1) at (0,-1.65) {};
\node[inner sep=0pt,align=center, scale=0.5] (hacker) at (-0.25,-1.65)
 {
$7$
};
\node[server, scale=0.8](s2) at (0.6,-1.65) {};
\node[inner sep=0pt,align=center, scale=0.5] (hacker) at (0.35,-1.65)
 {
$8$
};
\node[server, scale=0.8](s3) at (1.2,-1.65) {};
\node[inner sep=0pt,align=center, scale=0.5] (hacker) at (0.95,-1.65)
 {
$9$
};
\node[server, scale=0.8](s4) at (1.8,-1.65) {};
\node[inner sep=0pt,align=center, scale=0.5] (hacker) at (1.55,-1.65)
 {
$10$
};
    \node[server, scale=0.8](s5) at (2.4,-1.65) {};
\node[inner sep=0pt,align=center, scale=0.5] (hacker) at (2.15,-1.65)
 {
$11$
};
    \node[server, scale=0.8](s6) at (-0.6,-1.65) {};
\node[inner sep=0pt,align=center, scale=0.5] (hacker) at (-0.85,-1.65)
 {
$6$
};
    \node[server, scale=0.8](s7) at (-1.2,-1.65) {};
\node[inner sep=0pt,align=center, scale=0.5] (hacker) at (-1.45,-1.65)
 {
$5$
};
    \node[server, scale=0.8](s8) at (-1.8,-1.65) {};
\node[inner sep=0pt,align=center, scale=0.5] (hacker) at (-2.05,-1.65)
 {
$4$
};
    \node[server, scale=0.8](s9) at (-2.4,-1.65) {};
\node[inner sep=0pt,align=center, scale=0.5] (hacker) at (-2.65,-1.65)
 {
$3$
};
    \node[server, scale=0.8](s10) at (-3,-1.65) {};
\node[inner sep=0pt,align=center, scale=0.5] (hacker) at (-3.25,-1.65)
 {
$2$
};

    \draw[-, color=black] (-3.4,-1) to (2.4,-1);
    \draw[-, color=black] (-3.2,-1) to (-3.2,-1.15);
    \draw[-, color=black] (-2.6,-1) to (-2.6,-1.15);
    \draw[-, color=black] (-2,-1) to (-2,-1.15);
    \draw[-, color=black] (-1.4,-1) to (-1.4,-1.15);
    \draw[-, color=black] (-0.8,-1) to (-0.8,-1.15);
    \draw[-, color=black] (-0.2,-1) to (-0.2,-1.15);
    \draw[-, color=black] (0.4,-1) to (0.4,-1.15);
    \draw[-, color=black] (1,-1) to (1,-1.15);
    \draw[-, color=black] (1.6,-1) to (1.6,-1.15);
    \draw[-, color=black] (2.2,-1) to (2.2,-1.15);
\draw[-, color=black] (-0.75,-1) to (-0.75,-0.62);

\draw[-, color=black] (-3.2,-1.9) to (-3.2,-2.3);
\node[server, scale=0.8](s11) at (-3,-2.65) {};
\node[inner sep=0pt,align=center, scale=0.5] (hacker) at (-3.25,-2.65)
 {
$12$
};
\node[rack switch, xshift=0.1cm,yshift=0.3cm, scale=0.6] at (-3,-3.75) (sw1){};

\draw[-, color=black] (-3.2,-2.9) to (-3.2,-3.3);

\node[server, scale=0.8](s11) at (-3.8,-4.3) {};
\node[inner sep=0pt,align=center, scale=0.5] (hacker) at (-4.05,-4.3)
 {
$13$
};
\node[server, scale=0.8](s11) at (-3.2,-4.3) {};
\node[inner sep=0pt,align=center, scale=0.5] (hacker) at (-3.45,-4.3)
 {
$14$
};
\node[server, scale=0.8](s11) at (-2.6,-4.3) {};
\node[inner sep=0pt,align=center, scale=0.5] (hacker) at (-2.85,-4.3)
 {
$15$
};
\node[server, scale=0.8](s11) at (-2,-4.3) {};
\node[inner sep=0pt,align=center, scale=0.5] (hacker) at (-2.25,-4.3)
 {
$16$
};
\draw[-, color=black] (-4.2,-3.7) to (-2,-3.7);
\draw[-, color=black] (-4,-3.7) to (-4,-3.85);
\draw[-, color=black] (-3.4,-3.7) to (-3.4,-3.85);
\draw[-, color=black] (-2.8,-3.7) to (-2.8,-3.85);
\draw[-, color=black] (-2.2,-3.7) to (-2.2,-3.85);
\draw[-, color=black] (-3.2,-3.7) to (-3.2,-3.55);

\draw[-, color=black] (-1.8,-1.95) to (-1.8,-2.4);
\node[rack switch, xshift=0.1cm,yshift=0.3cm, scale=0.6] at (-1.5,-2.8) (sw1){};
\draw[-, color=black] (-1,-2.2) to (-1,-3.6);
\node[server, scale=0.8](s11) at (-0.4,-2.6) {};
\node[inner sep=0pt,align=center, scale=0.5] (hacker) at (-0.65,-2.6)
 {
$17$
};
\node[server, scale=0.8](s11) at (-0.4,-3.4) {};
\node[inner sep=0pt,align=center, scale=0.5] (hacker) at (-0.65,-3.4)
 {
$18$
};
\draw[-, color=black] (-1,-2.6) to (-0.78,-2.6);
\draw[-, color=black] (-1,-3.2) to (-0.78,-3.2);

\draw[-, color=black] (-1,-2.5) to (-1.2,-2.5);

\node[rack switch, xshift=0.1cm,yshift=0.3cm, scale=0.6] at (0.6,-2.8) (sw1){};

\node[server, scale=0.8](s11) at (1.7,-2.55) {};
\node[inner sep=0pt,align=center, scale=0.5] (hacker) at (1.45,-2.55)
 {
$19$
};
\node[server, scale=0.8](s11) at (1.7,-3.35) {};
\node[inner sep=0pt,align=center, scale=0.5] (hacker) at (1.45,-3.35)
 {
$21$
};
\node[server, scale=0.8](s11) at (1.7,-4.15) {};
\node[inner sep=0pt,align=center, scale=0.5] (hacker) at (1.45,-4.15)
 {
$23$
};
%%\node[server, scale=0.8](s11) at (1.7,-4.95) {};
\node[server, scale=0.8](s11) at (2.3,-2.55) {};
\node[inner sep=0pt,align=center, scale=0.5] (hacker) at (2.05,-2.55)
 {
$20$
};
\node[server, scale=0.8](s11) at (2.3,-3.35) {};
\node[inner sep=0pt,align=center, scale=0.5] (hacker) at (2.05,-3.35)
 {
$22$
};
\node[server, scale=0.8](s11) at (2.3,-4.15) {};
\node[inner sep=0pt,align=center, scale=0.5] (hacker) at (2.05,-4.15)
 {
$24$
};
%%\node[server, scale=0.8](s11) at (2.3,-4.95) {};

\draw[-, color=black] (1.1,-2.1) to (1.1,-4.45);

\draw[-, color=black] (0.9,-2.45) to (1.1,-2.45);
\draw[-, color=black] (-0.25,-2.45) to (0,-2.45);

\draw[-, color=black] (1.1,-2.45) to (1.32,-2.45);
\draw[-, color=black] (1.1,-3.25) to (1.32,-3.25);
\draw[-, color=black] (1.1,-4.05) to (1.32,-4.05);
%%\draw[-, color=black] (1.1,-4.85) to (1.32,-4.85);

\draw[-, color=black] (-0.4,-3.7) to (-0.4,-4);
\node[rack switch, xshift=0.1cm,yshift=0.3cm, scale=0.6] at (-0.3,-4.4) (sw1){};

\draw[-, color=black] (-1.2,-4.35) to (0.2,-4.35);

\node[server, scale=0.8](s11) at (-0.7,-5) {};
\node[inner sep=0pt,align=center, scale=0.5] (hacker) at (-0.95,-5)
 {
$25$
};
\node[server, scale=0.8](s11) at (0,-5) {};
\node[inner sep=0pt,align=center, scale=0.5] (hacker) at (-0.25,-5)
 {
$26$
};
\draw[-, color=black] (-0.9,-4.55) to (-0.9,-4.35);
\draw[-, color=black] (-0.2,-4.55) to (-0.2,-4.35);
\draw[-, color=black] (-0.55,-4.2) to (-0.55,-4.35);

\node[rack switch, xshift=0.1cm,yshift=0.3cm, scale=0.6] at (-2.4,-6.1) (sw1){};

\node[rack switch, xshift=0.1cm,yshift=0.3cm, scale=0.6] at (1.2,-6.1) (sw1){};

\draw[-, color=black] (-3,-5.5) to (-0.5,-5.5);

\draw[-, color=black] (-0.3,-5.5) to (1.3,-5.5);

\draw[-, color=black] (0,-5.5) to (0,-5.3);
\draw[-, color=black] (-0.7,-5.5) to (-0.7,-5.3);

\draw[-, color=black] (-2.5,-5.5) to (-2.5,-5.65);
\draw[-, color=black] (1,-5.5) to (1,-5.65);

\node[server, scale=0.8](s11) at (-2.3,-6.6) {};
\node[inner sep=0pt,align=center, scale=0.5] (hacker) at (-2.53,-6.6)
 {
$27$
};
\node[server, scale=0.8](s11) at (0.5,-6.6) {};
\node[inner sep=0pt,align=center, scale=0.5] (hacker) at (0.25,-6.6)
 {
$28$
};
\node[server, scale=0.8](s11) at (1.1,-6.6) {};
\node[inner sep=0pt,align=center, scale=0.5] (hacker) at (0.85,-6.6)
 {
$29$
};
\node[server, scale=0.8](s11) at (1.7,-6.6) {};
\node[inner sep=0pt,align=center, scale=0.5] (hacker) at (1.45,-6.6)
 {
$30$
};
\node[server, scale=0.8](s11) at (2.3,-6.6) {};
\node[inner sep=0pt,align=center, scale=0.5] (hacker) at (2.05,-6.6)
 {
$31$
};

\draw[-, color=black] (0,-6) to (2.4,-6);
\draw[-, color=black] (0.3,-6) to (0.3,-6.15);
\draw[-, color=black] (0.9,-6) to (0.9,-6.15);
\draw[-, color=black] (1.5,-6) to (1.5,-6.15);
\draw[-, color=black] (2.1,-6) to (2.1,-6.15);

\draw[-, color=black] (-2.5,-5.88) to (-2.5,-6.15);

\draw[-, color=black] (1,-5.88) to (1,-6);

%%\node[scale=0.1,opacity=1] (flag1) at (-2.55,-6.5)
%%{
%%  \begin{tikzpicture}
%%\draw[black,very thick, fill=Red] (0,1) to (0,3) to (2,3) to (0,4) to (0,1);
%%  \end{tikzpicture}
%%};
%%
%%\node[scale=0.1,opacity=1] (flag1) at (-3.25,-1.5)
%%{
%%  \begin{tikzpicture}
%%\draw[black,very thick, fill=Red] (0,1) to (0,3) to (2,3) to (0,4) to (0,1);
%%  \end{tikzpicture}
%%};
%%
%%\node[scale=0.1,opacity=1] (flag1) at (-3.25,-2.5)
%%{
%%  \begin{tikzpicture}
%%\draw[black,very thick, fill=Red] (0,1) to (0,3) to (2,3) to (0,4) to (0,1);
%%  \end{tikzpicture}
%%};
%%
%%\node[scale=0.1,opacity=1] (flag1) at (-2.05,-1.5)
%%{
%%  \begin{tikzpicture}
%%\draw[black,very thick, fill=Red] (0,1) to (0,3) to (2,3) to (0,4) to (0,1);
%%  \end{tikzpicture}
%%};
%%
%%\node[scale=0.1,opacity=1] (flag1) at (1.55,-1.5)
%%{
%%  \begin{tikzpicture}
%%\draw[black,very thick, fill=Red] (0,1) to (0,3) to (2,3) to (0,4) to (0,1);
%%  \end{tikzpicture}
%%};
%%
%%\node[scale=0.1,opacity=1] (flag1) at (-0.65,-2.5)
%%{
%%  \begin{tikzpicture}
%%\draw[black,very thick, fill=Red] (0,1) to (0,3) to (2,3) to (0,4) to (0,1);
%%  \end{tikzpicture}
%%};

%%\node[inner sep=0pt,align=center] (gpu1) at (1,0)
%%  {\scalebox{0.014}{
%%     \includegraphics{virus_logo_1.png}
%%   }
%% };
%%    \draw[-, color=black] (6.8,-0.45) to (6.8,-0.5);
%%\node[l3 switch, xshift=1.5cm,yshift=0.5cm, scale=0.8] at (4,-1) (s7){};
    \end{tikzpicture}
  };

      \end{tikzpicture}
};

%%
%%\node[scale=0.7] (kth_cr) at (6,-2.7)
%%{
%%  \begin{tikzpicture}
%%    \node[server, scale=0.8](s1) at (-3.1,-7.7) {};
%%  \node[inner sep=0pt,align=center, scale=0.7] (hacker) at (-3.1,-8.4)
%%  {Application\\
%%    server
%%  };
%%  \node[server, scale=0.8, color=Blue](s1) at (-1.7,-7.7) {};
%%  \node[inner sep=0pt,align=center, scale=0.7] (hacker) at (-1.7,-8.6)
%%  {Intrusion\\
%%    detection\\
%%    system
%%  };
%%
%%  \node[scale=0.15](R1) at (-0.5,-7.7) {\router{R1}};
%%
%%  \node[inner sep=0pt,align=center, scale=0.7] (hacker) at (-0.5,-8.3)
%%  {Gateway
%%  };
%%\node[rack switch, xshift=0.1cm,yshift=0.3cm, scale=0.6] at (1.1,-8.1) (sw1){};
%%
%%  \node[inner sep=0pt,align=center, scale=0.7] (hacker) at (0.9,-8.3)
%%  {Access\\switch
%%  };
%%
%%\node[scale=0.1,opacity=1] (flag1) at (1.9,-7.7)
%%{
%%  \begin{tikzpicture}
%%\draw[black,very thick, fill=Red] (0,1) to (0,3) to (2,3) to (0,4) to (0,1);
%%  \end{tikzpicture}
%%};
%%
%%  \node[inner sep=0pt,align=center, scale=0.7] (hacker) at (1.9,-8.2)
%%  {Flag
%%  };
%%
%%  \node[database, minimum width=0.5cm, minimum height=0.8cm,align=center] (s6) at (2.7,-7.7) {$\quad$};
%%
%%  \node[inner sep=0pt,align=center, scale=0.7] (hacker) at (2.7,-8.45)
%%  {Traffic\\generator
%%  };
%%\end{tikzpicture}
%%};

\end{tikzpicture}

%% file: tikz/method.tex
\begin{tikzpicture}[fill=white, >=stealth,
    node distance=3cm,
    database/.style={
      cylinder,
      cylinder uses custom fill,
      %%cylinder body fill=yellow!50,
      %%cylinder end fill=yellow!50,
      shape border rotate=90,
      aspect=0.25,
      draw}]

    \tikzset{
node distance = 9em and 4em,
sloped,
   box/.style = {%
    shape=rectangle,
    rounded corners,
    draw=blue!40,
    fill=blue!15,
    align=center,
    font=\fontsize{12}{12}\selectfont},
 arrow/.style = {%
    %%draw=blue!30,
    line width=0.1mm,% <-- select desired width
    -{Triangle[length=5mm,width=2mm]},
    shorten >=1mm, shorten <=1mm,
    font=\fontsize{8}{8}\selectfont},
}

%%\node[scale=0.8] (kth_cr) at (0,4.1)
%%{
%%  \begin{tikzpicture}
%%    \draw[-, color=black, fill=black!5] (-3.8,-6.38) to (0.8,-6.38) to (0.8,-4.6) to (-3.8, -4.6) to (-3.8,-6.38);
%%
%%      \end{tikzpicture}
%%};

\node[scale=0.8] (kth_cr) at (0,2.15)
{
  \begin{tikzpicture}
    \draw[-, color=black, fill=black!5] (-3.8,-6.38) to (0.8,-6.38) to (0.8,-4.6) to (-3.8, -4.6) to (-3.8,-6.38);

\node[scale=0.5] (level1) at (-1.5,-5.5)
{
  \begin{tikzpicture}
\node[draw,circle, minimum width=10mm, scale=0.7](s0) at (0,0) {\Large$s_{1,1}$};
\node[draw,circle, minimum width=10mm, scale=0.7](s1) at (2,0) {\Large$s_{1,2}$};
\node[draw,circle, minimum width=10mm, scale=0.7](s2) at (4,0) {\Large$s_{1,3}$};
\node[draw,circle, minimum width=10mm, scale=0.7](s3) at (6,0) {\Large$\hdots$};
\node[draw,circle, minimum width=10mm, scale=0.7](s4) at (8,0) {\Large$s_{1,n}$};

\draw[-{Latex[length=2mm]}, bend left] (s0) to (s1);
\draw[-{Latex[length=2mm]}, bend left] (s1) to (s0);
\draw[-{Latex[length=2mm]}, bend left] (s1) to (s2);
\draw[-{Latex[length=2mm]}, bend left] (s2) to (s1);
\draw[-{Latex[length=2mm]}, bend left] (s2) to (s3);
\draw[-{Latex[length=2mm]}, bend left] (s3) to (s2);
\draw[-{Latex[length=2mm]}, bend left] (s3) to (s4);
\draw[-{Latex[length=2mm]}, bend left] (s4) to (s3);

\node[draw,circle, minimum width=10mm, scale=0.7](s8) at (0,-1.2) {\Large$s_{2,1}$};
\node[draw,circle, minimum width=10mm, scale=0.7](s9) at (2,-1.2) {\Large$s_{2,2}$};
\node[draw,circle, minimum width=10mm, scale=0.7](s10) at (4,-1.2) {\Large$s_{2,3}$};
\node[draw,circle, minimum width=10mm, scale=0.7](s11) at (6,-1.2) {\Large$\hdots$};
\node[draw,circle, minimum width=10mm, scale=0.7](s12) at (8,-1.2) {\Large$s_{2,n}$};

\draw[-{Latex[length=2mm]}, bend left] (s8) to (s9);
\draw[-{Latex[length=2mm]}, bend left] (s9) to (s8);
\draw[-{Latex[length=2mm]}, bend left] (s9) to (s10);
\draw[-{Latex[length=2mm]}, bend left] (s10) to (s9);
\draw[-{Latex[length=2mm]}, bend left] (s10) to (s11);
\draw[-{Latex[length=2mm]}, bend left] (s11) to (s10);
\draw[-{Latex[length=2mm]}, bend left] (s11) to (s12);
\draw[-{Latex[length=2mm]}, bend left] (s12) to (s11);

\draw[-{Latex[length=2mm]}, bend left] (s8) to (s0);
\draw[-{Latex[length=2mm]}, bend left] (s0) to (s8);
\draw[-{Latex[length=2mm]}, bend left] (s9) to (s1);
\draw[-{Latex[length=2mm]}, bend left] (s1) to (s9);
\draw[-{Latex[length=2mm]}, bend left] (s10) to (s2);
\draw[-{Latex[length=2mm]}, bend left] (s2) to (s10);
\draw[-{Latex[length=2mm]}, bend left] (s11) to (s3);
\draw[-{Latex[length=2mm]}, bend left] (s3) to (s11);
\draw[-{Latex[length=2mm]}, bend left] (s12) to (s4);
\draw[-{Latex[length=2mm]}, bend left] (s4) to (s12);

\node[draw,circle, minimum width=10mm, scale=0.7](s16) at (0,-2.4) {\Large$\vdots$};
\node[draw,circle, minimum width=10mm, scale=0.7](s17) at (2,-2.4) {\Large$\vdots$};
\node[draw,circle, minimum width=10mm, scale=0.7](s18) at (4,-2.4) {\Large$\vdots$};
\node[draw,circle, minimum width=10mm, scale=0.7](s19) at (6,-2.4) {\Large$\vdots$};
\node[draw,circle, minimum width=10mm, scale=0.7](s20) at (8,-2.4) {\Large$\vdots$};

\draw[-{Latex[length=2mm]}, bend left] (s16) to (s8);
\draw[-{Latex[length=2mm]}, bend left] (s8) to (s16);

\draw[-{Latex[length=2mm]}, bend left] (s17) to (s9);
\draw[-{Latex[length=2mm]}, bend left] (s9) to (s17);

\draw[-{Latex[length=2mm]}, bend left] (s18) to (s10);
\draw[-{Latex[length=2mm]}, bend left] (s10) to (s18);

\draw[-{Latex[length=2mm]}, bend left] (s19) to (s11);
\draw[-{Latex[length=2mm]}, bend left] (s11) to (s19);

\draw[-{Latex[length=2mm]}, bend left] (s20) to (s12);
\draw[-{Latex[length=2mm]}, bend left] (s12) to (s20);

\draw[-{Latex[length=2mm]}, bend left] (s16) to (s17);
\draw[-{Latex[length=2mm]}, bend left] (s17) to (s16);
\draw[-{Latex[length=2mm]}, bend left] (s17) to (s18);
\draw[-{Latex[length=2mm]}, bend left] (s18) to (s17);
\draw[-{Latex[length=2mm]}, bend left] (s18) to (s19);
\draw[-{Latex[length=2mm]}, bend left] (s19) to (s18);
\draw[-{Latex[length=2mm]}, bend left] (s19) to (s20);
\draw[-{Latex[length=2mm]}, bend left] (s20) to (s19);

    \end{tikzpicture}
  };
    \end{tikzpicture}
  };

\node[scale=0.8] (kth_cr) at (0,-2.15)
{
  \begin{tikzpicture}
    \draw[-, color=black, fill=black!5] (-3.8,-6.38) to (0.8,-6.38) to (0.8,-4.6) to (-3.8, -4.6) to (-3.8,-6.38);

\node[inner sep=0pt,align=center, scale=0.8] (time) at (-1.5,-5.5)
{
  \begin{tikzpicture}
    \node[server](s1) at (0.5,0.4) {};

 \node[rack switch, xshift=0.1cm,yshift=0.3cm] at (3,0.4) (s5){};
    \node[database, minimum width=0.5cm, minimum height=0.6cm,fill=black!15,align=center] (s6) at (3,-0.6) {$\quad\quad$};
    \node[database, minimum width=0.5cm, minimum height=0.6cm,fill=black!15,align=center] (s7) at (4,-0.6) {$\quad\quad$};
    \node[database, minimum width=0.5cm, minimum height=0.6cm,fill=black!15,align=center] (s7) at (2,-0.6) {$\quad\quad$};

    \node[server](s8) at (0.5,-0.7) {};
    \node[server](s12) at (-0.5,-0.7) {};
    \node[server](s12) at (-0.5,0.4) {};

    \draw[-, color=black] (3,0.58) to (3,-0.25);
    \draw[-, color=black] (3,-0.05) to (2,-0.05) to (2, -0.25);
    \draw[-, color=black] (3,-0.05) to (4,-0.05) to (4, -0.25);
    \draw[-, color=black] (3,-0.05) to (-1,-0.05);
    \draw[-, color=black] (-0.6,-0.05) to (-0.6,0.02);
    \draw[-, color=black] (-0.8,-0.05) to (-0.8,-0.12);
    \draw[-, color=black] (0.4,-0.05) to (0.4,0.02);
    \draw[-, color=black] (0.2,-0.05) to (0.2,-0.12);
%%    \draw[-, color=black] (6.8,-0.45) to (6.8,-0.5);
%%\node[l3 switch, xshift=1.5cm,yshift=0.5cm, scale=0.8] at (4,-1) (s7){};
    \end{tikzpicture}
  };

      \end{tikzpicture}
    };

\node[scale=0.8] (kth_cr) at (0,0)
{
  \begin{tikzpicture}
    \draw[-, color=black, fill=black!5] (-3.8,-6.38) to (0.8,-6.38) to (0.8,-4.6) to (-3.8, -4.6) to (-3.8,-6.38);
    \end{tikzpicture}
  };

  \draw[->, color=black] (0.5, 0.75) to (0.5, 1.4);
  \draw[<-, color=black] (-0.5, 0.75) to (-0.5, 1.4);

  \draw[->, color=black] (0.5, -1.4) to (0.5, -0.75);
\draw[<-, color=black] (-0.5, -1.4) to (-0.5, -0.75);
%%\draw[->, color=black, dashed] (0, 2.85) to (0, 3.35);

%%\node[inner sep=0pt,align=center, scale=0.8] (time) at (-2.9,3.5)
%%{
%%\small \textit{Policy learning}
%%};

%%\node[inner sep=0pt,align=center, scale=0.9] (time) at (-2.8,4)
%%{
%%  \small $(s,a) \rightarrow (s^{\prime}, r)$\\
%%  \small  Simulation
%%%%\small w. Domain Randomization\\
%%};

\node[inner sep=0pt,align=center, scale=0.9] (time) at (-3.15,0)
{
\small \textsc{Emulation System}
%%\small  Computer \\
%%\small infrastructure
};

\node[inner sep=0pt,align=center, scale=0.9] (time) at (-3,-2.1)
{
\small \textsc{Target}\\\small \textsc{Infrastructure}
%%\small  Computer \\
%%\small infrastructure
};

\node[inner sep=0pt,align=center, scale=0.7] (time) at (1.3,1.1)
{
  \small \textit{Model} \\
  \small \textit{Estimation}
  %%\& \\
%%  \small \textit{System} \textit{Identification} \\
};

\node[inner sep=0pt,align=center, scale=0.7] (time) at (-1.4,1.05)
{
  \small \textit{Strategy} \textit{Mapping} \\
  $\pi$
};

\node[inner sep=0pt,align=center, scale=0.7] (time) at (1.2,-1.07)
{
  \small \textit{Selective}\\
  \small \textit{Replication}
};

\node[inner sep=0pt,align=center, scale=0.7] (time) at (-1.5,-1.07)
{
  \small \textit{Strategy}\\\small \textit{Implementation} $\pi$
%%  \small \textit{Automation}
};

%%\node[inner sep=0pt,align=center, scale=0.9] (time) at (1,3.1)
%%{
%%  \small \textit{Sampling}
%%};

\node[inner sep=0pt,align=center, scale=0.9] (time) at (-3.15,2.1)
{
%%\small $\mathcal{M}$\\
%%\small   Model
\small \textsc{Simulation System}
};

%%\node[scale=0.8] (kth_cr) at (0,3.5)
%%{
%%  \begin{tikzpicture}
%%    \draw[-, color=black, fill=black!5] (0,0) to (3.7,0) to (3.7,1) to (0, 1) to (0,0);
%%
%%      \end{tikzpicture}
%%};
%%
%%\node[inner sep=0pt,align=center, scale=0.9] (time) at (0,3.35)
%%{
%%\small $\sup_{\pi} \mathbb{E}_{\pi}\left[\sum_{t=1}^{\infty}\gamma^{t-1}r_t\right]$\\
%%};

%%\draw[->, color=black] (-1.7, 2.85) to (-1.7, 3.5) to (-1.5, 3.5);
%%\draw[->, color=black] (1.5, 3.5) to (1.7, 3.5) to (1.7, 2.85);

\draw[->, color=black, bend right=90, dashed] (1.85, 1.8) to (1.85,2.5);

\node[inner sep=0pt,align=center, scale=0.7] (time) at (3.3,2.2)
{
  \small Game model \&\\
  \small Reinforcement Learning
%%  \small \& Generalization
%%  \small \textit{Generalization}
%%  \small \textit{POMDP Model}
};

%\node[inner sep=0pt,align=center, scale=0.7] (time) at (-0.95,1.6)
%{
%\small \textit{MDP Model}
%};

%%\node[inner sep=0pt,align=center, scale=0.7] (time) at (0,-0.59)
%%{
%%\small \textit{Virtualized infrastructure}
%%};

%%\node[inner sep=0pt,align=center, scale=0.7] (time) at (0.95,1.6)
%%{
%%\small \textit{Model Realization}
%%};

%%\draw[<->, color=black] (-0.2, 2.1) to (0.2, 2.1);

%%\draw[->, color=black, dashed] (1.5, 3.6) to (4, 3.6) to (4, 1);

%%\node[inner sep=0pt,align=center, scale=1.5] (time) at (4.05,0.8)
%%{
%%\small $\pi$
%%};

\node[inner sep=0pt,align=center, scale=0.7] (time) at (3.07,0)
{
  \small Strategy evaluation \&\\
  \small Model estimation
};

\draw[->, color=black, bend right=90, dashed] (1.85, -0.4) to (1.85,0.3);

\draw[->, color=black, bend right=90, dashed] (1.85, -2.5) to (1.85,-1.8);

\node[inner sep=0pt,align=center, scale=0.7] (time) at (3.1,-2.15)
{
  \small Automated\\
  \small Intrusion Prevention
  %%\&\\
%%  \small Self-learning systems
};

%%\draw[->, color=black, dashed] (4, 0.6) to (4,0) to (1.9,0);

%%\node[inner sep=0pt,align=center, scale=0.3] (gpu1) at (4.5,0.85)
%%{
%%\begin{tikzpicture}[scale=1]
%% \tikzstyle{every node}=[font=\small]
%%
%% \draw [draw=black] (0.25,0.0) rectangle (0.5,0.15);
%% \draw [draw=black] (0.5,0.0) rectangle (0.75,0.3);
%% \draw [draw=black] (0.75,0.0) rectangle (1,0.6);
%% \draw [draw=black] (1,0.0) rectangle (1.25,0.9);
%% \draw [draw=black] (1.25,0.0) rectangle (1.5,0.3);
%% \draw [draw=black] (1.5,0.0) rectangle (1.75,0.9);
%% \draw [draw=black] (1.75,0.0) rectangle (2,0.6);
%% \draw [draw=black] (2,0.0) rectangle (2.25,0.9);
%% \draw [draw=black] (2.25,0.0) rectangle (2.5,0.15);
%%
%% \end{tikzpicture}
%%};

\node[inner sep=0pt,align=center, scale=0.6] (gpu1) at (0,0)
{
\begin{tikzpicture}[scale=1]
\draw grid (4, 1);
\draw [shift={(0, 0, 1)}] grid (4, 1);
\foreach \i in {0,...,1}\foreach \j in {0,...,4}\foreach \k in {0,1}
\node[scale=0.12] (n-\i-\j-\k) at (\j, \i, \k) {
\begin{tikzpicture}[x  = {(0.5cm,0.5cm)},
                    y  = {(1.5cm,-0.25cm)},
                    z  = {(0cm,0.9cm)}, draw]
\begin{scope}[canvas is yz plane at x=-1]
  \shade[left color=black!50,right color=black!20, draw] (-1,-1) rectangle (1,1);
  \node[inner sep=0pt,align=center] (hacker) at (0,0)
  {\scalebox{0.1}{
     \includegraphics{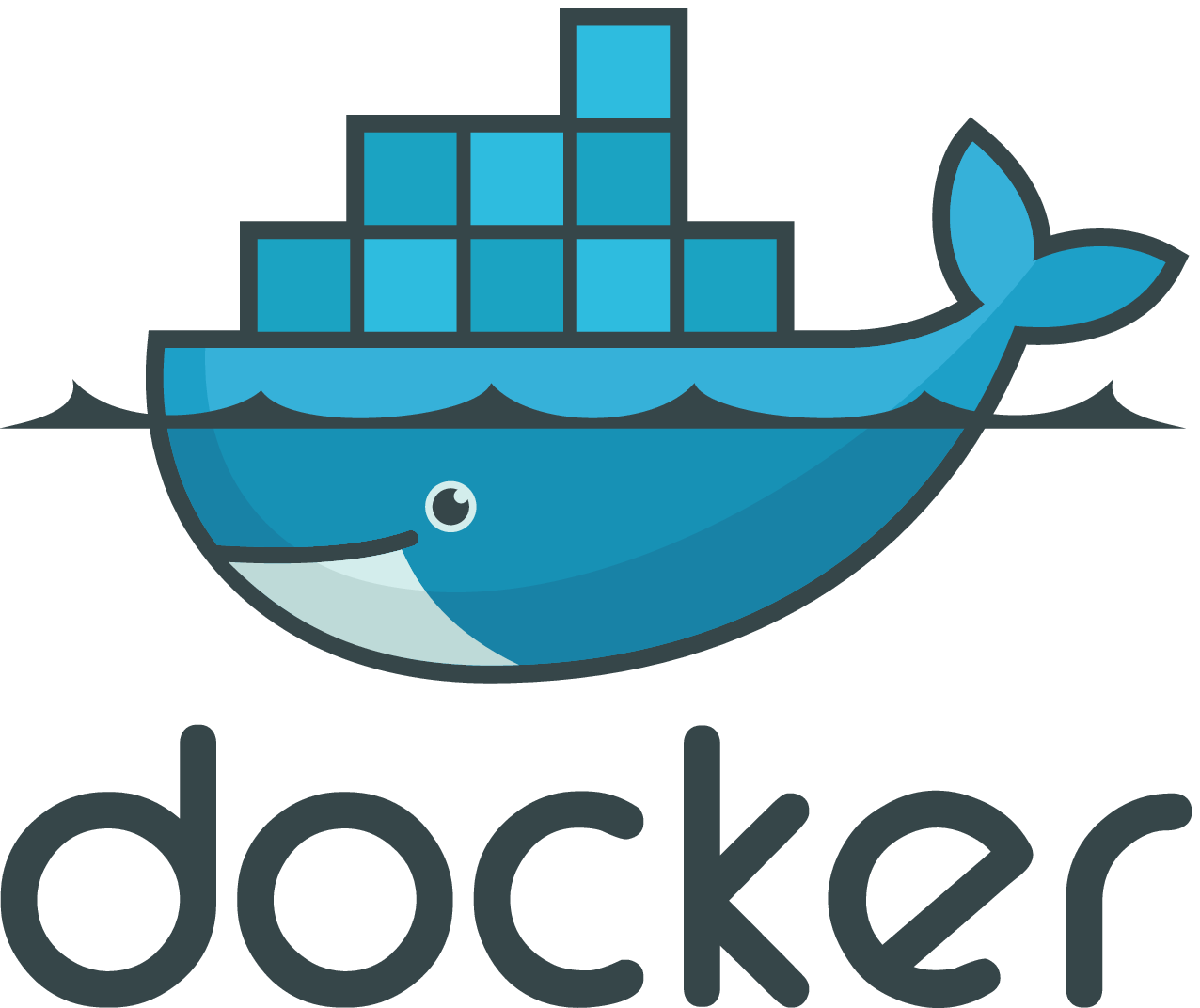}
   }
 };
\end{scope}
\begin{scope}[canvas is xz plane at y=1]
  \shade[right color=black!70,left color=black!20, draw] (-1,-1) rectangle (1,1);
\end{scope}
\begin{scope}[canvas is yx plane at z=1]
  \shade[top color=black!80,bottom color=black!20, draw] (-1,-1) rectangle (1,1);
\end{scope}
\end{tikzpicture}
};
%%\node[inner sep=0pt,align=center, scale=0.1] (gpu1) at (\j, \i, \k) {
%%t};
%\%foreach \i in {0,...,1}\foreach \j in {0,...,4}
% % \draw (n-\i-\j-0) -- (n-\i-\j-1);
 \end{tikzpicture}
};

%%\node[inner sep=0pt,align=center, scale=0.1] (gpu1) at (0,-3)
%%{
%%\begin{tikzpicture}[x  = {(0.5cm,0.5cm)},
%%                    y  = {(1.5cm,-0.25cm)},
%%                    z  = {(0cm,0.9cm)}, draw]
%%\begin{scope}[canvas is yz plane at x=-1]
%%  \shade[left color=black!50,right color=black!20, draw] (-1,-1) rectangle (1,1);
%%\end{scope}
%%\begin{scope}[canvas is xz plane at y=1]
%%  \shade[right color=black!70,left color=black!20, draw] (-1,-1) rectangle (1,1);
%%\end{scope}
%%\begin{scope}[canvas is yx plane at z=1]
%%  \shade[top color=black!80,bottom color=black!20, draw] (-1,-1) rectangle (1,1);
%%\end{scope}
%%\end{tikzpicture}
%%};
\end{tikzpicture}

%% file: tikz/state_transitions_9.tex
\begin{tikzpicture}[fill=white, >=stealth,
    node distance=3cm,
    database/.style={
      cylinder,
      cylinder uses custom fill,
      %%cylinder body fill=yellow!50,
      %%cylinder end fill=yellow!50,
      shape border rotate=90,
      aspect=0.25,
      draw}]

\node[scale=0.8] (kth_cr) at (0,2.15)
{
  \begin{tikzpicture}

\node[scale=1] (level1) at (-1.7,-5.6)
{
  \begin{tikzpicture}
\node[draw,circle, minimum width=15mm, scale=0.6](s0) at (0,0) {};
\node[draw,circle, minimum width=15mm, scale=0.6](s1) at (4,0) {};
\node[draw,circle, minimum width=15mm, scale=0.6](s2) at (2,-2.5) {};
\node[draw,circle, minimum width=15mm, scale=0.5](s4) at (2,-2.5) {};

\node[inner sep=0pt,align=center, scale=1.1] (time) at (0.07,0)
{
$0$
};

\node[inner sep=0pt,align=center, scale=1.1] (time) at (4.07,0)
{
$1$
};

\node[inner sep=0pt,align=center, scale=1.1] (time) at (2.07,-2.5)
{
$\emptyset$
};

\node[inner sep=0pt,align=center, scale=1.1] (time) at (-0.8,1.25)
{
$t\geq 1$\\
$l_t > 0$
};

\node[inner sep=0pt,align=center, scale=1.1] (time) at (4,1.25)
{
$t\geq 2$\\
$l_t>0$
};

\node[inner sep=0pt,align=center, scale=1.1] (time) at (2,0.5)
{
$a^{(2)}_t=S$
};

\node[inner sep=0pt,align=center, scale=1.1,rotate=0] (time) at (2.05,-1.2)
{
  $l_t=1$\\
  $a_t^{(1)}=S$
};

\node[inner sep=0pt,align=center, scale=1.1,rotate=55] (time) at (4.25,-1.35)
{
  prevented w.p $\phi_{l_t}$\\
  or stopped ($a^{(2)}_t=S$)
};

\draw[thick,-{Latex[length=2mm]}] (s0) to (s1);
\draw[thick,-{Latex[length=2mm]}] (s1) to (s2);
\draw[thick,-{Latex[length=2mm]}] (s0) to (s2);

\draw[thick,-{Latex[length=2mm]}, bend left=30] (s1) to (s2);

\draw[thick,-{Latex[length=2mm]}] (s0.90) arc (0:260:3.5mm);

\draw[thick,-{Latex[length=2mm]}] (s1.90) arc (0:260:3.5mm);

    \end{tikzpicture}
  };
    \end{tikzpicture}
  };

\end{tikzpicture}

%% file: tikz/threshold_policy_8.tex
\begin{tikzpicture}[fill=white, >=stealth,
    node distance=3cm,
    database/.style={
      cylinder,
      cylinder uses custom fill,
      %%cylinder body fill=yellow!50,
      %%cylinder end fill=yellow!50,
      shape border rotate=90,
      aspect=0.25,
      draw}]

    \tikzset{
node distance = 9em and 4em,
sloped,
   box/.style = {%
    shape=rectangle,
    rounded corners,
    draw=blue!40,
    fill=blue!15,
    align=center,
    font=\fontsize{12}{12}\selectfont},
 arrow/.style = {%
    %%draw=blue!30,
    line width=0.1mm,% <-- select desired width
    -{Triangle[length=5mm,width=2mm]},
    shorten >=1mm, shorten <=1mm,
    font=\fontsize{8}{8}\selectfont},
}

\node[scale=1] (system) at (0,0)
{
\begin{tikzpicture}
\draw[->, color=black] (0.0,0) to (6,0);

\node[inner sep=0pt,align=center, scale=0.8] (time) at (6.3,0)
{
  $b(1)$
};

\node[inner sep=0pt,align=center, scale=0.8] (time) at (0.05,-0.3)
{
$0$
};

\node[inner sep=0pt,align=center, scale=0.8] (time) at (5.75,-0.3)
{
$1$
};

\draw[-, color=black] (5.7,0.1) to (5.7,-0.1);

\draw[-, color=black] (0,0.1) to (0,-0.1);

\draw [decorate,decoration={brace,amplitude=5pt,mirror,raise=4pt},yshift=0pt,rotate=180, line width=0.20mm]
(-5.65,0.1) -- (-4.35,0.1) node [black,midway,xshift=0.1cm] {};

\node[inner sep=0pt,align=center, scale=0.8] (time) at (5.1,0.4)
{
$\mathscr{S}^{(1)}_{1,\pi_{2,l}}$
};

\draw [decorate,decoration={brace,amplitude=5pt,mirror,raise=4pt},yshift=0pt,rotate=180, line width=0.20mm]
(-5.65,-0.4) -- (-3.5,-0.4) node [black,midway,xshift=0.1cm] {};

\node[inner sep=0pt,align=center, scale=0.8] (time) at (4.8,0.9)
{
$\mathscr{S}^{(1)}_{2,\pi_{2,l}}$
};

\node[inner sep=0pt,align=center, scale=0.8] (time) at (4.1,1.25)
{
$\vdots$
};

\draw [decorate,decoration={brace,amplitude=5pt,mirror,raise=4pt},yshift=0pt,rotate=180, line width=0.20mm]
(-5.65,-1.18) -- (-2.3,-1.18) node [black,midway,xshift=0.1cm] {};

\node[inner sep=0pt,align=center, scale=0.8] (time) at (4.15,1.7)
{
$\mathscr{S}^{(1)}_{L,\pi_{2,l}}$
};

\draw[-, color=black] (4.3,0.1) to (4.3,-0.1);
\draw[-, color=black] (3.48,0.1) to (3.48,-0.1);

\draw[-, color=black] (2.35,0.1) to (2.35,-0.1);

\node[inner sep=0pt,align=center, scale=0.8] (time) at (4.3,-0.3)
{
$\tilde{\alpha}_{1}$
};
\node[inner sep=0pt,align=center, scale=0.8] (time) at (3.48,-0.3)
{
$\tilde{\alpha}_{2}$
};

\node[inner sep=0pt,align=center, scale=0.8] (time) at (2.35,-0.3)
{
$\tilde{\alpha}_{L}$
};

\node[inner sep=0pt,align=center, scale=0.8] (time) at (2.8,-0.3)
{
$\hdots$
};
\end{tikzpicture}
};

\node[scale=1] (system) at (0,-2.25)
{
\begin{tikzpicture}
\draw[->, color=black] (0.0,0) to (6,0);

\node[inner sep=0pt,align=center, scale=0.8] (time) at (6.3,0)
{
  $b(1)$
};

\node[inner sep=0pt,align=center, scale=0.8] (time) at (0.05,-0.3)
{
$0$
};

\node[inner sep=0pt,align=center, scale=0.8] (time) at (5.75,-0.3)
{
$1$
};

\draw[-, color=black] (5.7,0.1) to (5.7,-0.1);

\draw[-, color=black] (0,0.1) to (0,-0.1);

\draw [decorate,decoration={brace,amplitude=5pt,mirror,raise=4pt},yshift=0pt,rotate=180, line width=0.20mm]
(-5.65,0.1) -- (-4.35,0.1) node [black,midway,xshift=0.1cm] {};

\draw [decorate,decoration={brace,amplitude=5pt,mirror,raise=4pt},yshift=0pt,rotate=180, line width=0.20mm]
(-5.65,-0.5) -- (-3.2,-0.5) node [black,midway,xshift=0.1cm] {};

\node[inner sep=0pt,align=center, scale=0.8] (time) at (5,0.45)
{
$\mathscr{S}^{(2)}_{1,1,\pi_{1,l}}$
};
\node[inner sep=0pt,align=center, scale=0.8] (time) at (4.4,1.05)
{
$\mathscr{S}^{(2)}_{1,L,\pi_{1,l}}$
};

%\node[inner sep=0pt,align=center, scale=0.8] (time) at (4.1,1.25)
%{
%$\vdots$
%};

%\draw [decorate,decoration={brace,amplitude=5pt,mirror,raise=4pt},yshift=0pt,rotate=180, line width=0.20mm]
%(-5.65,-1.18) -- (-2.3,-1.18) node [black,midway,xshift=0.1cm] {};
%
%\node[inner sep=0pt,align=center, scale=0.8] (time) at (4.15,1.7)
%{
%$\mathscr{S}^{(2)}_{1,L}$
%};

\draw[-, color=black] (4.3,0.1) to (4.3,-0.1);
\draw[-, color=black] (3.25,0.1) to (3.25,-0.1);

%\draw[-, color=black] (2.35,0.1) to (2.35,-0.1);

\node[inner sep=0pt,align=center, scale=0.8] (time) at (4.3,-0.3)
{
$\tilde{\beta}_{1,1}$
};
\node[inner sep=0pt,align=center, scale=0.8] (time) at (3.28,-0.3)
{
$\tilde{\beta}_{1,L}$
};

%\node[inner sep=0pt,align=center, scale=0.8] (time) at (2.35,-0.3)
%{
%$\tilde{\beta}_{1,L}$
%};

\draw[-, color=black] (1.15,0.1) to (1.15,-0.1);

\node[inner sep=0pt,align=center, scale=0.8] (time) at (1.1,-0.3)
{
$\tilde{\beta}_{0,1}$
};

\node[inner sep=0pt,align=center, scale=0.8] (time) at (1.725,-0.3)
{
$\hdots$
};
\draw[-, color=black] (2.35,0.1) to (2.35,-0.1);
\node[inner sep=0pt,align=center, scale=0.8] (time) at (2.3,-0.3)
{
$\tilde{\beta}_{0,L}$
};

\node[inner sep=0pt,align=center, scale=0.8] (time) at (3.8,-0.3)
{
$\hdots$
};

\draw [decorate,decoration={brace,amplitude=5pt,mirror,raise=4pt},yshift=0pt,rotate=180, line width=0.20mm]
(-2.3,-0.5) -- (-0.05,-0.5) node [black,midway,xshift=0.1cm] {};

\draw [decorate,decoration={brace,amplitude=5pt,mirror,raise=4pt},yshift=0pt,rotate=180, line width=0.20mm]
(-1.1,0.1) -- (-0.05,0.1) node [black,midway,xshift=0.1cm] {};

\node[inner sep=0pt,align=center, scale=0.8] (time) at (0.55,0.45)
{
$\mathscr{S}^{(2)}_{0,1,\pi_{1,l}}$
};
\node[inner sep=0pt,align=center, scale=0.8] (time) at (1.1,1.05)
{
$\mathscr{S}^{(2)}_{0,L,\pi_{1,l}}$
};
\end{tikzpicture}
};

\end{tikzpicture}

%% file: tikz/fp_3.tex
      \begin{tikzpicture}[fill=white, >=stealth,
    node distance=3cm,
    database/.style={
      cylinder,
      cylinder uses custom fill,
      %%cylinder body fill=yellow!50,
      %%cylinder end fill=yellow!50,
      shape border rotate=90,
      aspect=0.25,
      draw}]

    \tikzset{
node distance = 9em and 4em,
sloped,
   box/.style = {%
    shape=rectangle,
    rounded corners,
    draw=blue!40,
    fill=blue!15,
    align=center,
    font=\fontsize{12}{12}\selectfont},
 arrow/.style = {%
    %%draw=blue!30,
    line width=0.1mm,% <-- select desired width
    -{Triangle[length=5mm,width=2mm]},
    shorten >=1mm, shorten <=1mm,
    font=\fontsize{8}{8}\selectfont},
}

\node[scale=1] (system) at (0,0)
{
\begin{tikzpicture}
\node[inner sep=0pt,align=center] (a1) at (0,0)
  {\scalebox{0.08}{
     \includegraphics{hacker.png}
   }
 };

  \node[inner sep=0pt,align=center, scale=0.65] (pi2_br) at (0.55,1.35)
  {
   $\tilde{\pi}_{2,l}\in B_2(\pi_{1,l})$
 };

  \node[inner sep=0pt,align=center, scale=0.65] (pi2) at (0,0.4)
  {
    $\pi_{2,l}$
  };
%  \draw[->, color=black, dashed] (0, 0.6) to (0, 1.2);

 \node[inner sep=0pt,align=center] (d1) at (0,-1.1)
  {\scalebox{0.04}{
     \includegraphics{laptop3.pdf}
   }
 };

  \node[inner sep=0pt,align=center, scale=0.65] (pi1) at (0,-1.45)
  {
   $\pi_{1,l}$
 };

  \node[inner sep=0pt,align=center, scale=0.65] (pi1_br) at (0.55,-2.4)
  {
   $\tilde{\pi}_{1,l}\in B_1(\pi_{2,l})$
 };
%  \draw[->, color=black, dashed] (0, -1.65) to (0, -2.25);
  \draw[->, color=black, bend left=30] (a1) to (d1);
  \draw[->, color=black, bend left=30] (d1) to (a1);

\node[rotate=-90](c1) at (0,-1.9) {
    \begin{tikzpicture}
\draw[->, x=0.07cm,y=0.05cm, black]
        (5,0) sin (6,-1) cos (7,0)
        sin (8,1) cos (9,0) sin (10,-1) cos (11,0) sin (12,1) cos (13, 0)
        to (14.5,0);
    \end{tikzpicture}
  };

\node[rotate=90](c1) at (0,0.85) {
    \begin{tikzpicture}
\draw[->, x=0.07cm,y=0.05cm, black]
        (5,0) sin (6,-1) cos (7,0)
        sin (8,1) cos (9,0) sin (10,-1) cos (11,0) sin (12,1) cos (13, 0)
        to (14.5,0);
    \end{tikzpicture}
  };

\end{tikzpicture}
};
\draw[->, color=black] (-0.2,0) to (1.4,0);
%  \node[inner sep=0pt,align=center, scale=0.65] (pi2_br) at (0.65,0.3)
%  {
%   $\pi^{\prime}_{2,l}=\pi_{2,l}\cup \tilde{\pi}_{2,l}$
% };
%
%  \node[inner sep=0pt,align=center, scale=0.65] (pi2_br) at (0.65,-0.3)
%  {
%   $\pi^{\prime}_{1,l}=\pi_{1,l}\cup \tilde{\pi}_{1,l}$
% };

\node[scale=1] (system) at (2.35,0)
{
\begin{tikzpicture}
\node[inner sep=0pt,align=center] (a1) at (0,0)
  {\scalebox{0.08}{
     \includegraphics{hacker.png}
   }
 };

  \node[inner sep=0pt,align=center, scale=0.65] (pi2_br) at (0.55,1.35)
  {
   $\tilde{\pi}^{\prime}_{2,l}\in B_2(\pi^{\prime}_{1,l})$
 };

  \node[inner sep=0pt,align=center, scale=0.65] (pi2) at (0,0.4)
  {
    $\pi^{\prime}_{2,l}$
  };

 \node[inner sep=0pt,align=center] (d1) at (0,-1.1)
  {\scalebox{0.04}{
     \includegraphics{laptop3.pdf}
   }
 };

  \node[inner sep=0pt,align=center, scale=0.65] (pi1) at (0,-1.45)
  {
   $\pi^{\prime}_{1,l}$
 };

  \node[inner sep=0pt,align=center, scale=0.65] (pi1_br) at (0.55,-2.45)
  {
   $\tilde{\pi}^{\prime}_{1,l}\in B_1(\pi^{\prime}_{2,l})$
 };

  \draw[->, color=black, bend left=30] (a1) to (d1);
  \draw[->, color=black, bend left=30] (d1) to (a1);
\node[rotate=-90](c1) at (0,-1.94) {
    \begin{tikzpicture}
\draw[->, x=0.07cm,y=0.05cm, black]
        (5,0) sin (6,-1) cos (7,0)
        sin (8,1) cos (9,0) sin (10,-1) cos (11,0) sin (12,1) cos (13, 0)
        to (14.5,0);
    \end{tikzpicture}
  };

\node[rotate=90](c1) at (0,0.89) {
    \begin{tikzpicture}
\draw[->, x=0.07cm,y=0.05cm, black]
        (5,0) sin (6,-1) cos (7,0)
        sin (8,1) cos (9,0) sin (10,-1) cos (11,0) sin (12,1) cos (13, 0)
        to (14.5,0);
    \end{tikzpicture}
  };

\end{tikzpicture}
};
\draw[->, color=black] (2.2,0) to (2.5,0);

  \node[inner sep=0pt,align=center, scale=1.5] (class) at (3.1,0)
  {
    $\hdots$
  };

  \draw[->, color=black] (3.5,0) to (3.8,0);

\node[scale=1] (system) at (4.3,0)
{
\begin{tikzpicture}
\node[inner sep=0pt,align=center] (a1) at (0,0)
  {\scalebox{0.08}{
     \includegraphics{hacker.png}
   }
 };

  \node[inner sep=0pt,align=center, scale=0.65] (class) at (0,0.45)
  {
    $\pi^{*}_{2,l} \in B_2(\pi_{1,l}^{*})$
  };

 \node[inner sep=0pt,align=center] (d1) at (0,-1.1)
  {\scalebox{0.04}{
     \includegraphics{laptop3.pdf}
   }
 };

  \node[inner sep=0pt,align=center, scale=0.65] (class) at (0,-1.5)
  {
    $\pi^{*}_{1,l} \in B_1(\pi_{2,l}^{*})$
  };
  \draw[->, color=black, bend left=30] (a1) to (d1);
  \draw[->, color=black, bend left=30] (d1) to (a1);
\end{tikzpicture}
};

\end{tikzpicture}

%% file: tikz/management_2.tex
      \begin{tikzpicture}[fill=white, >=stealth,
    node distance=3cm,
    database/.style={
      cylinder,
      cylinder uses custom fill,
      %%cylinder body fill=yellow!50,
      %%cylinder end fill=yellow!50,
      shape border rotate=90,
      aspect=0.25,
      draw}]

    \tikzset{
node distance = 9em and 4em,
sloped,
   box/.style = {%
    shape=rectangle,
    rounded corners,
    draw=blue!40,
    fill=blue!15,
    align=center,
    font=\fontsize{12}{12}\selectfont},
 arrow/.style = {%
    %%draw=blue!30,
    line width=0.1mm,% <-- select desired width
    -{Triangle[length=5mm,width=2mm]},
    shorten >=1mm, shorten <=1mm,
    font=\fontsize{8}{8}\selectfont},
}

\node[scale=0.8] (kth_cr) at (6,-5)
{
  \begin{tikzpicture}[auto, thick]

 % Cloud creation
  \node[cloud, fill=gray!20, cloud puffs=16, cloud puff arc= 100,
        minimum width=7cm, minimum height=2.5cm, aspect=1] at (0,0) {};

  \node[server, scale=0.8] (a1) at (-2.5,0.3) {};
  \node[server, scale=0.8] (a2) at (-1.75,-0.55) {};
  \node[server, scale=0.8] (a3) at (-1.2,0.55) {};
  \node[server, scale=0.8] (a4) at (-0.75,-0.7) {};
  \node[server, scale=0.8] (a5) at (-0.25,0) {};
  \node[server, scale=0.8] (a6) at (0.25,0.7){};
  \node[server, scale=0.8] (a7) at (0.75,-0.3) {};
%  \node[server, scale=0.8] (a8) at (1.5,0) {};
  \node[scale=0.13](a8) at (1.5,0) {\router{}};
\node[inner sep=0pt,align=center] (a9) at (2.5,-0.3)
  {\scalebox{0.12}{
     \includegraphics{hacker.png}
   }
 };

\node[inner sep=0pt,align=center, scale=0.8] (client) at (2.8,0.5)
  {\scalebox{0.07}{
     \includegraphics{laptop3.pdf}
   }
 };
%  \node[server, scale=0.8] (a9) at (2.5,0.4) {};

 % Physical layer links
  \path[thin] (a1) edge (a2);
  \path[thin] (a1) edge (a3);
  \path[thin] (a2) edge (a3);
  \path[thin] (a3) edge (a6);
  \path[thin] (a2) edge (a4);
  \path[thin] (a5) edge (a6);
  \path[thin] (a5) edge (a4);
  \path[thin] (a5) edge (a2);
  \path[thin] (a5) edge (a7);
  \path[thin] (a6) edge (a7);
%  \path[thin] (a6) edge (a9);
  \path[thin] (a6) edge (a8);
  \path[thin] (a8) edge (a9);
  \path[thin] (a7) edge (a8);
  \path[thin] (a8) edge (client);

%\node[scale=0.8] (kth_cr) at (-0.25,2.7)
%{
%  \begin{tikzpicture}[auto, thick]
%    \draw[black, thick] (0,0) to (1.25, 0) to (1.25,1.5);
%    \draw[black, thick] (0,0) to (0, 1.5);
%    \draw[black, thick] (0,0.25) to (1.25, 0.25);
%    \draw[black, thick] (0,0.5) to (1.25, 0.5);
%    \draw[black, thick] (0,0.75) to (1.25, 0.75);
%  \end{tikzpicture}
%};

%\draw[-{Latex[length=2mm]}] (-0.25,3) to (-0.25, 4);
%\draw[-{Latex[length=2mm]}] (-0.25,4) to (-0.25, 4.4);
\node[draw,circle, scale=2.5, fill=Red!20!black!20, align=center](stream_processor) at (-2,3.5) {};

\node[database, minimum width=1.5cm, minimum height=0.8cm,align=center, opacity=1, fill=Black!30] (stream_state) at (-2,2) {$\quad$};
\draw[-, dashed] (stream_processor) to (stream_state);

\draw[-{Latex[length=2mm]}] (-3.55,0) to (-3.85, 0) to (-3.85, 3.5) to (stream_processor);

\node[inner sep=0pt,align=center, scale=1] (hacker) at (-3,3.75)
 {
$o_t$
};

\draw[-{Latex[length=2mm]}] (stream_processor) to (0.2, 3.5);

%\node[inner sep=0pt,align=center, scale=1] (hacker) at (1.5,3.8)
% {
%($b_t,s_t,l_t$)
%};

%\draw[-{Latex[length=2mm]}] (-0.25,5.1) to (-0.25,

%\draw[-{Latex[length=2mm]}] (0.6,6.2) to (2.2, 6.2) to (2.2, 1);

\node[inner sep=0pt,align=center, scale=1] (hacker) at (2.9,3.75)
 {
$\bm{a}_t$
};

\node[inner sep=0pt,align=center, scale=1] (hacker) at (-0.5,3.75)
{
$b_t$
};

\node[inner sep=0pt,align=center, scale=1] (hacker) at (-1.95,2)
{
$h^{(1)}_t,h^{(2)}_t$
};

\draw[-{Latex[length=2mm]}] (1.85, 3.5) to (3.9, 3.5) to (3.9, 0) to (3.5, 0);

\node[draw,rectangle, scale=2.5, fill=Green!20!black!20, align=center, minimum height=0.2cm,minimum width=0.8cm](rl_box) at (1.25,3.5) {};
6);

\node[inner sep=0pt,align=center, scale=1] (rl) at (1.3,3.5)
 {
$\pi_{1,l}, \pi_{2,l}$
};
\node[database, minimum width=0.3cm, minimum height=0.8cm,align=center, opacity=1, fill=Black!30] (rl_state) at (1.25,2) {$\quad$};

\draw[-, dashed] (rl_box) to (rl_state);

\node[inner sep=0pt,align=center, scale=1] (hacker) at (1.3,2)
{
$s_t$
};
\end{tikzpicture}
};

\end{tikzpicture}

%% file: tikz/value_fun.tex
\begin{tikzpicture}[
    % define a style for the dots
    dot/.style={
        draw=black,
        fill=blue!90,
        circle,
        minimum size=3pt,
        inner sep=0pt,
        solid,
    },
    ]

\node[scale=1] (kth_cr) at (0,2.15)
{
  \begin{tikzpicture}
    \begin{axis}[
      xmin=0,
      grid=major,
      grid style={dashed},
      xmax=1,
        ymin=-0.6,
        ymax=0.1,
        axis lines=center,
       ticks=none,
        xlabel style={below right},
        ylabel style={above left},
        axis line style={-{Latex[length=2mm]}},
        smooth,
        legend style={at={(0.55,0.72)}},
        legend columns=2,
        legend style={
                    % the /tikz/ prefix is necessary here...
                    % otherwise, it might end-up with `/pgfplots/column 2`
                    % which is not what we want. compare pgfmanual.pdf
            /tikz/column 2/.style={
                column sep=5pt,
              }
            },
            width=12.5cm,
            height=5.5cm
        ]

%\addplot[black,mark=diamond,mark repeat=10,mark size=1.3pt,samples=100,smooth, name path=l0, domain=0:1]   (x,{0});

\addplot[Blue,mark=square,mark repeat=10,mark size=1.3pt,samples=100,smooth, name path=l1, domain=0:0.235]   (x,{(1-x)*-0.244144 + x*-1.27207});

\addplot[Red,mark=triangle,mark repeat=10,mark size=1.3pt,samples=100,smooth, name path=l2,  domain=0:0.24]   (x,{(1-x)*-0.264144 + x*-1.27207});

\addplot[Blue,mark=square,mark repeat=10,mark size=1.3pt,samples=100,smooth, name path=l1, domain=0.235:0.98]   (x,{(1-x)*-0.63033 + x*0});
\addplot[Red,mark=triangle,mark repeat=10,mark size=1.3pt,samples=100,smooth, name path=l2,  domain=0.24:0.98]   (x,{(1-x)*-0.67033 + x*0});

%%\addplot[Red,mark=none,smooth, name path=l1, thick, domain=0:0.97]   (x,{(x+1)/2});
%%\addplot[Red,mark=none,smooth, name path=l1, thick, domain=0.49:1]   (x,0);
%%\addplot[Red,dashed,mark=none,smooth, name path=l1, thick,domain=0:1]   (0.5,x);
\legend{$\hat{V}^{*}_{7}(b(1))$, $\hat{V}^{*}_{1}(b(1))$}
\end{axis}
%%\node[inner sep=0pt,align=center, scale=1, rotate=0, opacity=1] (obs) at (6,4.5)
%%{
%%$1$-dimensional unit simplex
%%};
\node[inner sep=0pt,align=center, scale=1, rotate=0, opacity=1] (obs) at (10.6,3.6)
{
  $1$
};

\node[inner sep=0pt,align=center, scale=1, rotate=0, opacity=1] (obs) at (11.4,3.4)
{
  $b(1)$
};

\node[inner sep=0pt,align=center, scale=1, rotate=0, opacity=1] (obs) at (-0.3,3.4)
{
  $0$
};

\node[inner sep=0pt,align=center, scale=1, rotate=0, opacity=1] (obs) at (-0.5,2)
{
  $-0.25$
};

\end{tikzpicture}
};

  \end{tikzpicture}